\documentclass[10pt]{article}

\usepackage[preprint]{tmlr}
\usepackage{hyperref}
\usepackage{url}
\usepackage{amsmath,amssymb}
\usepackage{amsthm}
\newtheorem{proposition}{Proposition}
\newtheorem{corollary}[proposition]{Corollary}
\usepackage{booktabs}
\usepackage{array}
\usepackage{multirow}
\usepackage{graphicx}

\title{A Symmetry-Matching Approach to Blind-Spot Reduction\\
in Sparse Autoencoders}

\newcommand{\openreview}{\url{https://openreview.net/forum?id=NWWpKC9CZH}}

\author{\name Csaba Balogh \email balogh.csaba@okfo.gov.hu \\
      \addr National Directorate General for Hospitals, Budapest, Hungary \\
      \addr Data Science in Health Programme, Semmelweis University,
      Budapest, Hungary \\[1ex]
      {\normalfont\small Published in Transactions on Machine Learning Research (08/2026).} \\
      {\normalfont\small\textbf{OpenReview:} \openreview}}

\begin{document}

\maketitle

%=======================================================================
\begin{abstract}
Language models can treat semantically distinct inputs as interchangeable
at the representation level, creating blind spots that standard sparse
autoencoder (SAE) objectives do not target. In safety-critical
settings such as clinical dosage extraction, legal clause interpretation,
or financial amount verification, such failures can propagate silently
into downstream decisions. We study this problem as one of feature-basis
orientation. Building on a symmetry-matching view from algebraic
error-detection theory, we add a differentiable $V_{\mathrm{Gini}}$
regularisation term to SAE training that penalises uneven perturbation
response within a specified family.
Across GPT-2, Gemma~2, and Qwen~2.5, a frozen multi-seed evaluation shows
that $V$-regularisation raises both relative and absolute lower-tail
responses while reducing response-profile imbalance, with no material
general-text reconstruction degradation. JumpReLU and minimum description length (MDL)
objectives remain close to the Standard SAE profile under the same
protocol. Family-geometry controls show that the effect is not explained
by template duplication or a single dominant perturbation family, and
held-out probes plus null-calibrated separation tests evaluate how far the
representation-level change carries to downstream discriminability. A
zero-shot audit on out-of-distribution clinical radiology minimal pairs
finds the relative weak-tail lift in all evaluated model-by-family cells,
not perfectly rank-aligned with the trained statistic, and a held-out
family audit finds it on family kinds never used in training.
Together, the results show that
perturbation-aware SAE training can reorient feature bases toward
safety-relevant input changes while leaving downstream validation as a
separate empirical step.
\end{abstract}

%=======================================================================
\section{Introduction}
\label{sec:intro}

Consider a clinical pipeline that must preserve the distinction between
``the patient received 3~doses'' and ``the patient received 8~doses''.
In the learned feature space of a sparse autoencoder, such a numeric
change can land in the weak lower tail of the perturbation-response
profile, so the code moves very little for exactly the pairs where the
change matters. For the dose value itself this weakness stays hidden,
because the digit sits on the surface of the text and any downstream
readout can recover it directly. The same weak-tail geometry has no such
safety net for distinctions that no single surface token carries, such as
negation, severity, or laterality, and there it can translate into
degraded downstream access. This is not primarily a failure of knowledge.
It is a failure of orientation in the feature space. The relevant
distinction may be present in the hidden state, but poorly arranged in
the feature basis, so that the perturbation induces too little
code-space separation to support reliable access.

The same algebraic phenomenon appears in a seemingly unrelated domain. A
decimal checksum that uses weighted modular arithmetic cannot detect
certain adjacent or jump transpositions, because the arithmetic makes
those transpositions invisible. The Hungarian patient identifier (TAJ)
and several US healthcare identifiers (NPI, DEA) are built on exactly such
checksums and inherit this blind spot \citep{b_general}. The fix in both
cases is not more redundancy but better orientation of the existing
encoding.

The need for systematic perturbation-awareness testing in clinical AI is
increasingly recognised. \citet{b_moradi} demonstrated that
state-of-the-art clinical NLP models degrade under minor input
perturbations that simulate real-world noise in clinical text, and the
SemEval NLI4CT shared task \citep{b_nli4ct} attracted over 100
participants to evaluate biomedical inference robustness under controlled
perturbations. The major AI developers acknowledge the problem.
OpenAI's safety evaluation framework explicitly aims to ``protect against
blind spots'' through third-party assessments \citep{b_openai_safety}, and
Anthropic's interpretability programme has shown that SAEs can recover
safety-relevant features while also noting that the extracted feature set
is incomplete \citep{b_templeton}. A joint
Anthropic--OpenAI alignment evaluation exercise in 2025
\citep{b_anthropic_openai} evaluated public models from both developers
for harmful-action propensities in simulated settings and highlighted the
value of cross-developer testing for surfacing evaluation blind spots. The
present paper provides a family-specific,
differentiable measure of within-family response-profile imbalance, embeds
it in a broader perturbation-response hierarchy for evaluating perturbation
awareness, and turns that component into an SAE training signal.

This paper connects the clinical, algebraic, and interpretability domains
and argues that reliable artificial intelligence requires two independent
conditions. The first is \emph{knowledge}, addressed by the scaling-laws
literature \citep{b_kaplan,b_hoffmann} through larger models and more data.
The second is \emph{perturbation awareness}, the ability to respond to
input changes that matter for the task. We start from the
symmetry-matching condition $\operatorname{Stab}(G,F) \cap E = \{e\}$ of
the algebraic framework for error-detecting codes over finite alphabets
\citep{b_general}, and Section~\ref{sec:measure} gives a self-contained
account of the condition and its continuous adaptation. In continuous SAE
representations the response profile has more structure than a binary
visible/invisible distinction, so we evaluate response magnitude,
lower-tail behaviour, and the Gini imbalance of the finite response
profile.

We add $V_{\mathrm{Gini}}$ as a differentiable regularisation term to the
sparse autoencoder (SAE) training objective and compare the result against
three alternative SAE objectives (standard $L_1$, JumpReLU
\citep{b_jumprelu}, and MDL (minimum description length) \citep{b_mdl})
on three language models of different scale and architecture. We then
extend the evaluation from six perturbation families to sixteen, add
family-geometry controls, and test whether
representation-level changes transfer to held-out and null-calibrated
probes.

\paragraph{Contributions.}
(i) We make representational blind spots measurable and localisable by
defining a perturbation-response hierarchy for continuous SAE
representations, connecting global exact invariance, sampled
zero-response sets, low-response regions,
lower-tail response magnitudes, and finite response-profile imbalance.
This makes perturbation awareness measurable across the hierarchy and
response-profile imbalance directly optimisable
(\S\ref{sec:measure}).
(ii) We introduce $V_{\mathrm{Gini}}$ regularisation as a differentiable
perturbation-sensitivity conditioning term for SAE training
(\S\ref{sec:training}).
(iii) We evaluate the method under a frozen multi-seed protocol and show
that it raises lower-tail perturbation responses while reducing
sensitivity imbalance (\S\ref{sec:results}).
(iv) We show that three standard SAE objectives ($L_1$, JumpReLU, and
MDL) do not reduce perturbation-response imbalance. All three leave the
profile close to the unregularised Standard SAE, establishing that
perturbation awareness is not captured by reconstruction, sparsity, or
description-length pressure. We add family-geometry, held-out, and
null-calibrated validation checks
(\S\ref{sec:results}--\S\ref{sec:discussion}).
(v) We use a unified extraction protocol to ensure that perturbation
responses are compared at matched sequence positions across GPT-2,
Gemma~2, and Qwen~2.5.

\paragraph{Related work.}
Sparse autoencoders have become a widely used tool for decomposing
language-model activations into interpretable features
\citep{b_bricken,b_cunningham,b_templeton,b_gao}, building on earlier work
on circuits and superposition \citep{b_olah,b_elhage}. A growing literature
questions whether the recovered features are canonical units of analysis
\citep{b_noncanonical,b_unified,b_fa}. This work is
orthogonal to that debate. Rather than asking whether the features are the
``right'' ones, we ask whether the basis they span is oriented so that
task-relevant perturbations are uniformly visible. The adversarial
robustness literature \citep{b_goodfellow,b_madry} measures the
\emph{magnitude} of perturbation needed to change a model's output; the sampled response profile instead identifies the
perturbations to which a representation under-responds, while
$V_{\mathrm{Gini}}$ summarises their within-family imbalance, complementing the geometric/equivariance
view of deep networks \citep{b_bronstein}.

%=======================================================================
\section{Blind sets, perturbation responses, and sensitivity imbalance}
\label{sec:measure}

In this paper, a \emph{representational blind spot} is a task-relevant
input change to which the encoding gives insufficient response under a
stated evaluation criterion. We use \emph{perturbation awareness} for the
graded capacity of an encoding to respond to such task-relevant changes,
evaluated across the hierarchy below rather than by any single quantity.
Global stabiliser membership is the exact invariance case, while sampled
zero responses, margin sets, and lower-tail magnitudes provide operational
response levels. Response-profile imbalance, measured by
$V_{\mathrm{Gini}}$, is a property of a family of such responses rather
than of a single input change, and downstream accessibility is evaluated
separately.

%-----------------------------------------------------------------------
\subsection{The symmetry-matching condition}
\label{sec:measure-stab}

Every encoding has symmetries, meaning input transformations that leave
the encoded state unchanged. In a checksum these are undetected
transcription errors; in a neural representation they are input changes
that the representation does not separate. We recall the algebraic
condition that makes this notion precise and then adapt it to a continuous
encoder.

Following \citet{b_general}, let $G$ be a group acting on the input space
$X$, with the action written $g\cdot x$ for $g\in G$ and $x\in X$. The
elements of $G$ are the input transformations available to the channel.
Let $F\colon X\to Z$ be the encoding, here the sparse-autoencoder encoder
that maps a hidden state to its code, and let $E\subseteq G$ be the family
of task-relevant transformations, the changes that a downstream task
requires the representation to track. The stabiliser of the encoding is
the set of transformations under which the code does not move,
\begin{equation}
\operatorname{Stab}(G,F)
=
\lbrace g\in G \mid F(g\cdot x)=F(x)\ \text{on the evaluated domain}\rbrace.
\label{eq:stab_def}
\end{equation}
A transformation in $\operatorname{Stab}(G,F)$ is invisible to the
encoding. The symmetry-matching condition requires that no task-relevant
transformation be invisible,
\begin{equation}
\operatorname{Stab}(G,F)\cap E=\lbrace e\rbrace,
\label{eq:stab_condition}
\end{equation}
where $e$ is the identity transformation. When the intersection contains a
non-identity element, that element is a task-relevant change the encoding
cannot see, and it is the exact global case of a
representational blind spot. In the algebraic source of
Eq.~\eqref{eq:stab_condition} the group is finite and the stabiliser is a
finite set of sequence transformations, so the condition is decidable by
enumeration. The remainder of this section carries the same condition into
the continuous setting, where $F$ is a trained encoder and exact
invariance is the boundary case of a graded response.

The experiments in this paper do not test every transformation on the
whole evaluated domain. They use the finite observed pair family
\[
E_{\mathrm{obs}}
:=
\lbrace (x_i,g_i\cdot x_i)\rbrace_{i=1}^{K},
\qquad g_i\in E.
\]
Every quantity below is defined on these sampled pairs. The global
algebraic condition~\eqref{eq:stab_condition} and its sampled analogue
must therefore be kept distinct.

%-----------------------------------------------------------------------
\subsection{From exact invariance to a response profile}
\label{sec:measure-profile}

For each sampled pair $i$, with codes $z_i^{\mathrm{orig}}=F(x_i)$ and
$z_i^{\mathrm{pert}}=F(g_i\cdot x_i)$, define the relative SAE-code
response
\begin{equation}
D_i =
\frac{\|z_i^{\mathrm{pert}} - z_i^{\mathrm{orig}}\|}
     {\|z_i^{\mathrm{orig}}\|+\varepsilon}.
\label{eq:response}
\end{equation}
The response $D_i$ measures how far the code moves under the applied
transformation, normalised by the code norm so that the quantity is
comparable across pairs. The exact sampled zero-response set is
\begin{equation}
S_0(F,E_{\mathrm{obs}})
=
\lbrace i \mid F(g_i\cdot x_i)=F(x_i)\rbrace
=
\lbrace i \mid D_i=0\rbrace.
\label{eq:exact_blind}
\end{equation}
This is a pointwise sampled analogue of the global algebraic stabiliser
condition. If $g\in\operatorname{Stab}(G,F)$, every evaluated pair
generated by $g$ has zero response. The converse need not follow from a
single sampled equality $F(g_i\cdot x_i)=F(x_i)$, which is only a
pointwise collision.

In continuous learned representations, exact equality is too brittle to
serve as the operational endpoint: the empirical question is whether the
response is sufficiently small, not whether two floating-point code
vectors are exactly identical. We therefore retain the exact sampled set
as the exact-response boundary case corresponding to the algebraic notion,
and use margin and lower-tail quantities for operational evaluation. Two
graded quantities implement this operational level. For a fixed response
margin $\gamma$, the low-response margin set collects the pairs with
response below the margin,
\begin{equation}
S_\gamma(F,E_{\mathrm{obs}})=\lbrace i \mid D_i<\gamma\rbrace,
\label{eq:approx_blind}
\end{equation}
and the coverage statistic reports the mass above the margin,
\begin{equation}
C_\gamma(F,E_{\mathrm{obs}})
=\frac{1}{K}\sum_{i=1}^{K}\mathbf{1}[D_i\geq \gamma].
\label{eq:coverage}
\end{equation}
As $\gamma\to 0$ the set $S_\gamma$ contracts to $S_0$, so the margin
version is a continuous relaxation of the sampled zero-response set, with
$\gamma$
setting the detection scale. For margin-free summaries we use the lower
tail $L_{20}(D)$, the mean of the smallest quintile of responses, and the
absolute code displacement lower tail $L_{20}(\|\Delta z\|)$. The
lower-tail quantities report the magnitude of the weakest responses,
which captures the operational weak-response
regime relevant to blind-spot analysis. For a finite evaluated family
this interpretation is exact: $L_{20}(D)$ is the minimum, over all
subsets having lower-quintile cardinality, of their mean response
(Appendix~\ref{app:lower-tail-guarantee}).

%-----------------------------------------------------------------------
\subsection{The Gini imbalance of the response profile}
\label{sec:measure-gini}

The margin set and the lower tail measure the size of weak responses. The
regulariser acts on a complementary property, the inequality of the
response profile across the family. For a family of responses
$(D_1,\ldots,D_K)$ with mean $\bar D=K^{-1}\sum_i D_i$, define the
imbalance as the mean absolute pairwise difference of responses,
normalised by twice the mean,
\begin{equation}
V_{\mathrm{Gini}}(D)=
\frac{1}{2K^2\bar D}\sum_{i=1}^{K}\sum_{j=1}^{K}\lvert D_i-D_j\rvert.
\label{eq:gini_pairwise}
\end{equation}
This form states directly what the metric measures. The numerator sums
the response gap over every pair of transformations in the family, so it
is large when a few directions absorb most of the response and the rest
are weak, and it is zero when every direction responds equally. The
normalisation by $\bar D$ removes the response scale, so the metric
reports the shape of the profile rather than its size.

Equation~\eqref{eq:gini_pairwise} is the Gini coefficient of the response
profile, and it has an equivalent sorted form that is cheaper to evaluate.
Writing $D_{(1)}\leq\cdots\leq D_{(K)}$ for the sorted responses,
\begin{equation}
V_{\mathrm{Gini}}(D)
=
\frac{2\sum_{i=1}^{K} i\,D_{(i)}}{K\sum_{i=1}^{K} D_{(i)}}
-
\frac{K+1}{K}.
\label{eq:gini}
\end{equation}
The two forms are equal for every nonnegative profile. The sorted form is
used in training because it costs one sort rather than $K^2$ pairwise
differences, and its two terms are not a design choice. They follow from
the pairwise numerator by counting how often each response appears. Once
the responses are sorted, the ordered value $D_{(i)}$ is larger than the
$i-1$ responses below it and smaller than the $K-i$ responses above it.
In the one-sided sum over unordered pairs its signed coefficient is
$(i-1)-(K-i)=2i-(K+1)$. The double sum in
Eq.~\eqref{eq:gini_pairwise} counts every unordered pair twice, hence
\[
\sum_{i=1}^{K}\sum_{j=1}^{K}\lvert D_i-D_j\rvert
=2\sum_{i=1}^{K}\bigl(2i-(K+1)\bigr)D_{(i)}
=4\sum_{i=1}^{K} i\,D_{(i)}
-2(K+1)\sum_{i=1}^{K} D_{(i)}.
\]
Dividing by $2K\sum_i D_{(i)}$ produces the two terms of
Eq.~\eqref{eq:gini}. The first term is therefore the rank-weighted mean
of the sorted responses, which grows when large responses sit at high
ranks and small responses at low ranks, that is when the profile is
unequal. The second term $\tfrac{K+1}{K}$ is the value the first term
takes on a perfectly equal profile, so the subtraction sets the metric to
zero at equality and lets it grow toward one as the profile concentrates.
The doubled signed-rank coefficient is the reason the metric takes a
difference of two terms rather than a single sum.
Appendix~\ref{app:diff} records the boundary behaviour of the identity
and the differentiable surrogate used for training.

Two properties fix the interpretation. First, $V_{\mathrm{Gini}}(D)=0$ if
and only if all responses are equal, and it approaches its maximum when
the response concentrates on a single transformation, so it is a measure
of concentration, not of magnitude. Second, it is invariant to common
rescaling of the profile,
\begin{equation}
V_{\mathrm{Gini}}(cD_1,\ldots,cD_K)
=
V_{\mathrm{Gini}}(D_1,\ldots,D_K)
\quad\text{for all } c>0,
\label{eq:gini_scale}
\end{equation}
which follows because the factor $c$ cancels between numerator and
denominator in Eq.~\eqref{eq:gini_pairwise}. A representation can
therefore have a low $V_{\mathrm{Gini}}$ and still respond weakly in
absolute terms if every response is uniformly small. The scale-invariant
statistic is the training target because it is the quantity a
differentiable objective can shape, and the lower-tail quantities
$L_{20}(D)$ and $L_{20}(\|\Delta z\|)$ carry the magnitude information
that Eq.~\eqref{eq:gini_scale} discards. We report the two together
throughout, and we do not read a low $V_{\mathrm{Gini}}$ on its own as
evidence of detectability. For calibrated separation checks, we report
area under the receiver operating characteristic curve (AUROC). The
hierarchy is summarised in Table~\ref{tab:metric_hierarchy}.

\begin{table}[t]
\centering
\caption{Hierarchy of perturbation-response quantities used in this paper.}
\label{tab:metric_hierarchy}
\small
\begin{tabular}{@{}lll@{}}
\toprule
Object & Measures & Role \\
\midrule
$\operatorname{Stab}(G,F)\cap E=\lbrace e\rbrace$ & Absence of global task-relevant invariance & Algebraic reference condition \\
$D_i$ & Per-pair relative code displacement & Base sampled response \\
$S_0$ & Exact sampled zero-response set & Sampled exact-response boundary \\
$S_\gamma$ & Low-response margin set & Operational response threshold \\
$C_\gamma$ & Coverage above a fixed margin & Mass above a chosen threshold \\
$V_{\mathrm{Gini}}$ & Response-profile imbalance & SAE training objective \\
$L_{20}(D)$ & Lower-tail relative response & Weak-response lifting \\
$L_{20}(\|\Delta z\|)$ & Lower-tail absolute code displacement & Scale-sensitive support \\
Held-out probes & Discriminability evidence & Downstream transfer check \\
Null-calibrated AUROC & Separation from nuisance controls & Calibrated representation check \\
\bottomrule
\end{tabular}
\end{table}

%=======================================================================
\section{\texorpdfstring{$V$}{V}-regularised SAE training}
\label{sec:training}

The augmented training loss is
\begin{equation}
\mathcal{L} = \|h - \hat{h}\|^2
  + \lambda_1 \|z\|_1
  + \lambda_2\,\bar{V}_{\mathrm{Gini}}(z, E),
\label{eq:loss}
\end{equation}
where $\bar{V}_{\mathrm{Gini}}$ averages over sampled perturbation
families. Perturbation hidden states are pre-cached (one forward pass
through the frozen language model per perturbation pair) so that the
$V$ term adds negligible training cost and does not introduce stochastic
variation across steps (Appendix~\ref{app:lambda}). The base model's
weights remain frozen; only the SAE feature basis is modified.

%=======================================================================
\section{Experimental setup}
\label{sec:setup}

\paragraph{Models.}
GPT-2 small \citep{b_radford} (124M parameters, hidden dim 768,
layer~12), Gemma~2 2B \citep{b_gemma} (2.6B, hidden dim 2{,}304,
layer~13), and Qwen~2.5 3B \citep{b_qwen} (3B, hidden dim 2{,}048,
layer~18). All SAEs use $d_{\mathrm{sae}} = 4{,}096$ with ReLU activation
(standard and $V$-regularised variants), JumpReLU activation (JumpReLU
variant), or ReLU with an MDL penalty (MDL variant).

\paragraph{Hidden-state extraction.}
For decoder-only models with different padding conventions, hidden states
are extracted using a true-last-token rule computed from the attention
mask. This ensures that perturbation responses are compared at the final
non-padding token position across GPT-2, Gemma~2, and Qwen~2.5.
Reconstruction quality is measured on OpenWebText (OWT) hidden-state
samples as the normalised mean squared error (NMSE) between the SAE input
and its reconstruction. When raw reconstruction MSE is reported, it is
labelled explicitly.

\paragraph{Perturbation families.}
We evaluate six task-relevant perturbation families. Some are adapted
from classical edit-error taxonomies \citep{b_damerau,b_pollock}, while
others test semantic substitutions, negation, synonym changes, and
numeric swaps. The six families are semantic substitution, typo,
negation, synonym, number swap, and word order. Each contains 50
sentence pairs. For the generalisation study (\S\ref{sec:res-joint16}) we
add ten auto-generated medical-domain families (anatomical direction,
body-part swap, causal reversal, condition negation, date/time change,
drug-name swap, frequency change, gender swap, severity change, unit of
measure) with 30 pairs each, giving sixteen families and 600 pairs total.
A further pre-constructed pool of one hundred families with thirty fixed
pairs each, of which eighty-four were never used in training, supports the
held-out family audit (\S\ref{sec:res-e3}, Appendix~\ref{app:unseen}).
Representative examples for all sixteen families are listed in
Appendix~\ref{app:families}, and the complete datasets are provided as
supplementary material (see \nameref{sec:repro}).

\paragraph{Training protocol.}
All variants share $\lambda_1 = 10^{-3}$, learning rate $3 \times
10^{-4}$, batch size 64, 5{,}000 steps, and the Adam optimiser. The
$V$-regularised variants use $\lambda_2 = 0.1$, sampling 3 families per
step with 8 pairs each from the pre-cached perturbation hidden states.
The sixteen-family ``joint'' runs use the script \texttt{train\_joint.py}
with all sixteen families in the $V$ loss at every step, $\lambda_2=0.2$,
and 7{,}500 steps. All experiments run on a single Apple M-series GPU.

%=======================================================================
\section{Results}
\label{sec:results}

\subsection{Multi-seed evaluation across three model families}
\label{sec:res-e7}

We first report a frozen multi-seed stability evaluation. It uses
true-last-token perturbation caches for all three models. Table~\ref{tab:e7}
reports seed-aggregated deltas for Standard versus $V$-regularised SAEs,
and Figure~\ref{fig:response-cdf} visualises how the weakest Standard
responses move under $V$-regularisation.
With this hierarchy in place, we interpret a lower Gini value together
with the movement of the weak-response tail. Thus, we report the decrease
in $V_{\mathrm{Gini}}$ alongside positive $\Delta L_{20}(D)$ and positive
$\Delta L_{20}(\|\Delta z\|)$.
Across GPT-2, Gemma~2, and Qwen~2.5, $V$-regularisation raises lower-tail
relative response and lower-tail absolute code displacement while reducing
response-profile imbalance. The positive-family fraction is 1.0 in all
three models, so the effect is not driven by a small subset of families.

\begin{table}[t]
\centering
\caption{Frozen multi-seed Standard versus $V$-regularised SAE comparison.
Values are mean $\pm$ standard deviation across three seeds. Positive
$\Delta L_{20}(D)$ and $\Delta L_{20}(\|\Delta z\|)$ indicate lower-tail
response lifting; negative $\Delta V_{\mathrm{Gini}}$ indicates reduced
response-profile imbalance. The OWT reconstruction change is reported as
NMSE difference.}
\label{tab:e7}
\begin{center}
\small
\begin{tabular}{@{}lccccc@{}}
\toprule
Model & $\Delta L_{20}(D)$ & $\Delta L_{20}(\|\Delta z\|)$ &
$\Delta V_{\mathrm{Gini}}$ & Positive families & $\Delta$OWT NMSE \\
\midrule
GPT-2 & $+0.0562 \pm 0.0010$ & $+5.3584 \pm 0.1202$ &
$-0.2919 \pm 0.0008$ & $1.0000 \pm 0.0000$ & $+0.0000 \pm 0.0000$ \\
Gemma~2 & $+0.1301 \pm 0.0042$ & $+10.3454 \pm 1.1553$ &
$-0.2322 \pm 0.0001$ & $1.0000 \pm 0.0000$ & $-0.0004 \pm 0.0007$ \\
Qwen~2.5 & $+0.1316 \pm 0.0006$ & $+2.4650 \pm 0.1129$ &
$-0.2432 \pm 0.0001$ & $1.0000 \pm 0.0000$ & $+0.0006 \pm 0.0011$ \\
\bottomrule
\end{tabular}
\end{center}
\end{table}

\begin{figure}[t]
\centering
\includegraphics[width=\linewidth]{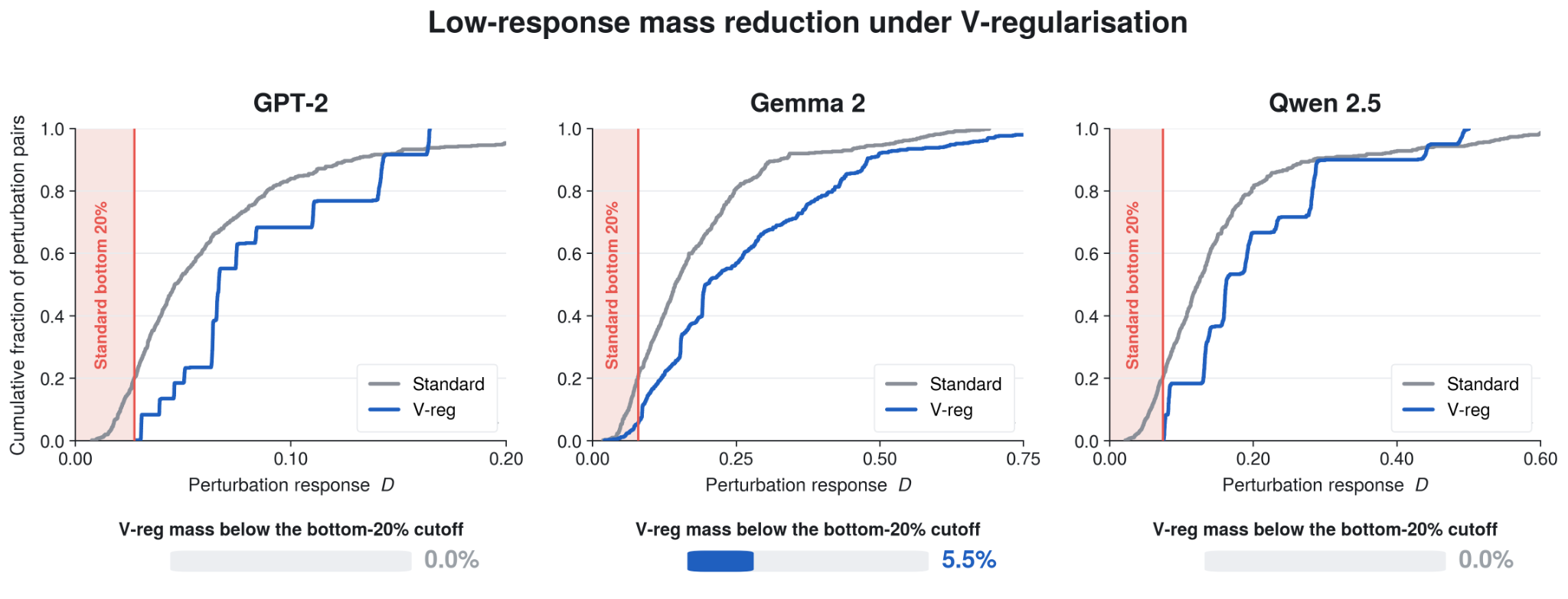}
\caption{Empirical cumulative distribution functions (CDFs) of per-pair
perturbation response $D$ over all
sixteen perturbation families. For each model, the red line marks the 20th
percentile of the Standard SAE response distribution, defining a
model-specific low-response region. Lower CDF values to the left of the
cutoff indicate fewer perturbation pairs remaining in the weak-response
region under $V$-regularisation. The step structure of the V-reg curves
reflects the resolution of each model's code space. GPT-2 small concentrates pairs on
fewer discrete response levels, producing visible steps, while the
higher-dimensional Gemma~2 spreads responses almost continuously.}
\label{fig:response-cdf}
\end{figure}

Figure~\ref{fig:response-cdf} gives the same result at the distributional
level. $V$-regularisation shifts the empirical response distribution to the
right, removing all or most perturbation pairs from the Standard
bottom-20\% response region. The remaining mass falls from 20\% by
construction to 0.0\% on GPT-2, 5.5\% on Gemma~2, and 0.0\% on Qwen~2.5.
Thus, the primary result is a three-model response profile effect,
lower-tail lifting together with reduced within-family imbalance and no
material OWT reconstruction degradation.

\subsection{Lower-tail response and absolute code displacement}
\label{sec:res-e1}

The absolute-response analysis in Table~\ref{tab:e1-absolute} reports the
magnitude side of the hierarchy. It gives the lower tail of relative
SAE-code response, the lower tail of absolute code displacement,
decoded-response support, and input-normalised gain on the sixteen-family
perturbation set.

\begin{table}[t]
\centering
\caption{Lower-tail response and absolute code displacement on the sixteen-family
perturbation set. Values are V-reg minus Standard for the corresponding
checkpoint pair. Positive $\Delta L_{20}(D)$ and
$\Delta L_{20}(\|\Delta z\|)$ indicate that the weak lower tail is lifted
in relative and absolute SAE-code response. The gain-normalised lower
tail records the input-complexity-normalised variant.}
\label{tab:e1-absolute}
\begin{center}
\small
\begin{tabular}{@{}lccccc@{}}
\toprule
Model & $\Delta L_{20}(D)$ & Positive families &
$\Delta L_{20}(\|\Delta z\|)$ & $\Delta L_{20}(\mathrm{decoded})$ &
$\Delta L_{20}(g)$ \\
\midrule
GPT-2 & $+0.055845$ & $16/16$ & $+5.4349$ & $+0.000368$ & $-0.159311$ \\
Gemma~2 & $+0.133151$ & $16/16$ & $+11.0187$ & $+0.067886$ & $-0.151532$ \\
Qwen~2.5 & $+0.131023$ & $16/16$ & $+2.4443$ & $+0.040904$ & $-0.107768$ \\
\bottomrule
\end{tabular}
\end{center}
\end{table}

Across all three models, the lower tail of relative response and the
lower tail of absolute code displacement increase in all sixteen families.
The decoded-response column is positive for Gemma~2 and Qwen~2.5 and close
to zero for GPT-2, so the representation-level lift has different decoded
expression across models. The input-normalised gain column separates this
main effect from a stricter gain-normalised variant.

\subsection{Family geometry and homogeneity controls}
\label{sec:res-e5}

The sixteen-family result could be misleading if it were driven by
duplicated templates, one dominant family, or unusually homogeneous
perturbation geometry. We therefore built exact-skeleton template clusters
for all 600 perturbation pairs and recomputed uncertainty with a
fixed-family template-cluster bootstrap. The collection contains 424
template clusters; 250 of 600 pairs (41.7\%) belong to repeated templates.

\begin{table}[t]
\centering
\caption{Template-cluster and family-resampled uncertainty for the
lower-tail effects. Fixed intervals resample template clusters within
each of the sixteen families. Family-resampled intervals additionally
resample families and are wider, but remain positive for the relative
lower-tail effect.}
\label{tab:e5}
\begin{center}
\small
\begin{tabular}{@{}lccc@{}}
\toprule
Model & Metric & Fixed-family interval & Family-resampled interval \\
\midrule
Gemma~2 & $\Delta L_{20}(D)$ & $+0.1332\,[+0.1256,+0.1361]$ & $[+0.0891,+0.1811]$ \\
GPT-2 & $\Delta L_{20}(D)$ & $+0.0558\,[+0.0527,+0.0567]$ & $[+0.0417,+0.0699]$ \\
Qwen~2.5 & $\Delta L_{20}(D)$ & $+0.1310\,[+0.1248,+0.1328]$ & $[+0.0868,+0.1792]$ \\
\midrule
Gemma~2 & $\Delta L_{20}(\|\Delta z\|)$ & $+11.0187\,[+9.8028,+11.7425]$ & $[+6.2071,+16.3177]$ \\
GPT-2 & $\Delta L_{20}(\|\Delta z\|)$ & $+5.4349\,[+5.1541,+5.5791]$ & $[+4.0521,+6.9244]$ \\
Qwen~2.5 & $\Delta L_{20}(\|\Delta z\|)$ & $+2.4443\,[+2.1110,+2.7675]$ & $[+0.9824,+4.1303]$ \\
\bottomrule
\end{tabular}
\end{center}
\end{table}

The relative lower-tail effect remains positive under every
leave-one-template-out deletion across 424/424 template clusters for all
three models, and leave-one-family-out macro effects remain positive. The
hidden perturbation geometry is well-defined under the true-last-token
extraction rule for every model and family. GPT-2 shows a more
low-dimensional perturbation geometry than Gemma~2 or Qwen~2.5, but this
geometric difference does not remove the positive sparse-code lower-tail
effect.

\subsection{Extrinsic endpoints beyond the training objective}
\label{sec:res-e3}

The objective optimises a scale-invariant imbalance statistic, so a
reduction in $V_{\mathrm{Gini}}$ alone does not establish that
perturbations become more accessible. This subsection therefore reports
validation endpoints beyond the optimised Gini statistic. Some remain
representation-level but are scale-sensitive or evaluated on an external
domain, while the probe endpoints test downstream accessibility directly.

A null-calibrated separation test measures code-space distance against
meaning-preserving nuisance controls rather than a zero-distance null. The
primary endpoint is the family-macro AUROC difference between
$V$-regularised and Standard SAEs on absolute code distance
(Table~\ref{tab:e2}). Gemma~2 and Qwen~2.5 show positive separation under
the fixed-family interval, whereas GPT-2 is uncertain in point direction.
The endpoint is a separation task the objective never sees, so it is
external to the framework that defines the training target.

Held-out probes test whether the representation-level change reaches a
semantic distinction on templates unseen during training. On the
held-out negation probe the distinction is fully present in the Qwen~2.5
hidden state at an AUROC of $1.00$, the Standard SAE code loses part of it
at $0.85$, and the $V$-regularised code recovers it to $0.91$. This is the
accessibility pattern the weak-tail signature predicts, closed on an
endpoint outside the objective, and it appears where the Standard code has
room to improve. On held-out negation templates the paired lower-tail
probability-margin point estimates also improve for all three models,
rising from 0.0052 to 0.3627 on GPT-2, from 0.1590 to 0.3646 on Gemma~2,
and from 0.0846 to 0.1440 on Qwen~2.5, with family-stratified bootstrap
intervals that include zero, so the margin version is suggestive rather
than uniform. A nested, template-grouped five-fold cross-fit on all 78
Qwen negation pairs further separates the Standard-defined weak quintile
from the remaining pairs. The Standard probability margin is lower on the
weak set, $0.302$ versus $0.590$, with a template-bootstrap
weak-minus-non-weak interval of $[-0.509,-0.066]$. On that fixed weak set
the Standard-to-$V$-regularised margin change is $+0.055$
[$-0.087,+0.257$], and correctness changes from $15/16$ to $16/16$;
the interval and exact McNemar test ($p=1.00$) do not resolve an
intervention effect. The weak/non-weak association is descriptive because
Standard $D_i$ and Standard margin share a representation source.
Appendix~\ref{app:lower-tail-guarantee} reports both subsets.

A real-clinical audit evaluates the frozen general checkpoints zero-shot
on OpenI chest radiography minimal pairs in four clinical families,
negation, laterality, severity, and anatomical direction. The lower-tail
response $L_{20}(D)$ rises in all eight model-by-family cells, with
report-clustered intervals excluding zero in each, at macro means of
$+0.0023$ on GPT-2 and $+0.0037$ on Qwen~2.5. The audit covers GPT-2 and
Qwen~2.5; Gemma~2 enters the external OpenI evaluation through the
laterality probe below and through the held-out family audit of
Appendix~\ref{app:unseen}, where every Gemma~2 family is individually
interval-positive. On this data the lower-tail lift and the trained
inequality statistic are not perfectly rank-aligned. The largest Qwen lift,
negation at $+0.0099$, accompanies a $V_{\mathrm{Gini}}$ change of only
$-0.003$, while the largest GPT-2 inequality drop, anatomical direction at
$-0.056$, accompanies the smallest lift at $+0.0008$. Because
$V_{\mathrm{Gini}}$ is scale-invariant by Eq.~\eqref{eq:gini_scale}, the
regulariser directly targets profile shape, while changes in mean or
absolute response are not themselves rewarded by the objective; parameter
updates can nevertheless change them indirectly. Appendix~
\ref{app:lower-tail-guarantee} gives the exact interpretation and paired
panel, and Appendix~\ref{app:casestudy} gives the own-tail case study.

Two controlled designs isolate the mechanism. A dedicated paired
experiment trains Standard and $V$-regularised SAEs from identical
initialisation and an identical reconstruction minibatch stream, differing
only in the $V$ term. The regulariser drives the trained family's response
inequality to near zero, from $0.233$ to $0.0009$ on GPT-2 and from
$0.113$ to $0.0008$ on Qwen~2.5, with no measurable reconstruction cost on
held OpenWebText tokens, where explained variance is unchanged on GPT-2
and rises from $0.995$ to $0.996$ on Qwen~2.5
(Appendix~\ref{app:dedicated}). A controlled synthetic experiment then
places the task in a non-saturated regime by construction. There the
$V$-regularised code yields a modest held-out probe gain of $+0.016$ in
AUROC together with a clear lower-tail displacement gain of $+0.51$, and
the probe gain shrinks as the Standard code approaches ceiling, matching
the headroom pattern of the real models (Appendix~\ref{app:toy}).

External OpenI laterality probes give a stricter boundary check. On
left/right minimal pairs, the lower-tail probability-margin deltas are
small, at $+0.0022$ for GPT-2, $-0.0747$ for Gemma~2, and $+0.0141$ for
Qwen~2.5. The intervals cross zero for GPT-2 and Qwen, while the Gemma~2
delta is negative with a report-clustered interval of $[-0.163, -0.012]$
that excludes zero, a downstream cost on this saturated external task. A targeted
certificate-gated repair on the same task moves the intended laterality
direction without a measurable downstream gain, consistent with this
saturated boundary (Appendix~\ref{app:repair}). Finally, a
held-out family audit evaluates the same frozen checkpoints on
seventy-two family kinds never used in training, drawn from a
pre-constructed pool and verifiable as held out from the released
training materials, and finds the lower-tail lift at the same
order of magnitude, including on eighteen family kinds with no mechanical
overlap with training (Appendix~\ref{app:unseen}).

The checks give a consistent reading. $V$-regularisation improves the
geometry of weak response directions on endpoints outside the objective,
and transfer to downstream probability margins depends on the task and
model, appearing where the Standard code has headroom.

\begin{table}[t]
\centering
\caption{Null-calibrated separation. Family-macro AUROC difference
(V-reg minus Standard) for absolute SAE-code distance against
meaning-preserving nuisance controls. Fixed-family intervals resample
template clusters within each family; family-resampled intervals also
resample families. Intervals are 95\% confidence intervals (CIs). Gemma~2
and Qwen~2.5 are positive under the fixed-family interval; GPT-2 is
uncertain in point direction.}
\label{tab:e2}
\begin{center}
\small
\begin{tabular}{@{}lccc@{}}
\toprule
Model & $\Delta$AUROC (abs.\ code) & Fixed-family 95\% CI &
Family-resampled 95\% CI \\
\midrule
Gemma~2 & $+0.037$ & $[+0.013, +0.063]$ & $[-0.014, +0.087]$ \\
Qwen~2.5 & $+0.048$ & $[+0.024, +0.077]$ & $[+0.007, +0.091]$ \\
GPT-2 & $-0.010$ & $[-0.049, +0.024]$ & $[-0.078, +0.064]$ \\
\bottomrule
\end{tabular}
\end{center}
\end{table}

\subsection{Response-profile imbalance under alternative SAE objectives}
\label{sec:res-alt}

We trained five SAE variants on each of the three models and evaluated
the $V_{\mathrm{Gini}}$ of each on the six perturbation families. The five
variants are the standard SAE ($L_1$ sparsity), JumpReLU SAE (learnable
threshold with an $L_0$ proxy, \citealp{b_jumprelu}), MDL SAE
(description-length penalty, \citealp{b_mdl}), $V$-regularised SAE
($V_{\mathrm{Gini}}$ penalty on raw response), and $V$-regularised
SAE with input normalisation ($V_{\mathrm{Gini}}$ penalty on
complexity-controlled gain).

Appendix~\ref{app:objective-comparison} reports the single-run numerical
detail. On all three models, the standard, JumpReLU, and MDL SAEs produce
similar $V$ profiles; the differences among these three methods are
consistently below 0.02 across all families and models. None of them
explicitly targets perturbation awareness, and under the shared protocol
they leave both lower-tail response and response-profile imbalance close
to Standard. In that descriptive comparison, the $V$-regularised SAE
produces substantially lower response-profile imbalance than the non-$V$
objectives, so reconstruction, sparsity, and description-length objectives
do not by construction condition the perturbation-response profile.

A frozen multi-seed baseline extension confirms this separation beyond a
single run. Using the same true-last-token caches and seeds as
\S\ref{sec:res-e7}, we trained JumpReLU and MDL across three seeds and
compared them with the Standard and $V$-regularised checkpoints under the
primary metric hierarchy (Table~\ref{tab:e7b}). JumpReLU and MDL track the
Standard objective on all three models, moving $L_{20}(D)$ and
$V_{\mathrm{Gini}}$ by less than $0.001$, while $V$-regularisation raises
the lower-tail response and reduces imbalance by a far larger margin and
hits a positive-family fraction of 1.0. None of the non-perturbation-aware
baselines reproduces the joint pattern of lower-tail lifting and
imbalance reduction. Figure~\ref{fig:baseline-separation} visualises this
separation in the two quantities most directly tied to the central claim.

\begin{figure}[t]
\centering
\includegraphics[width=\linewidth]{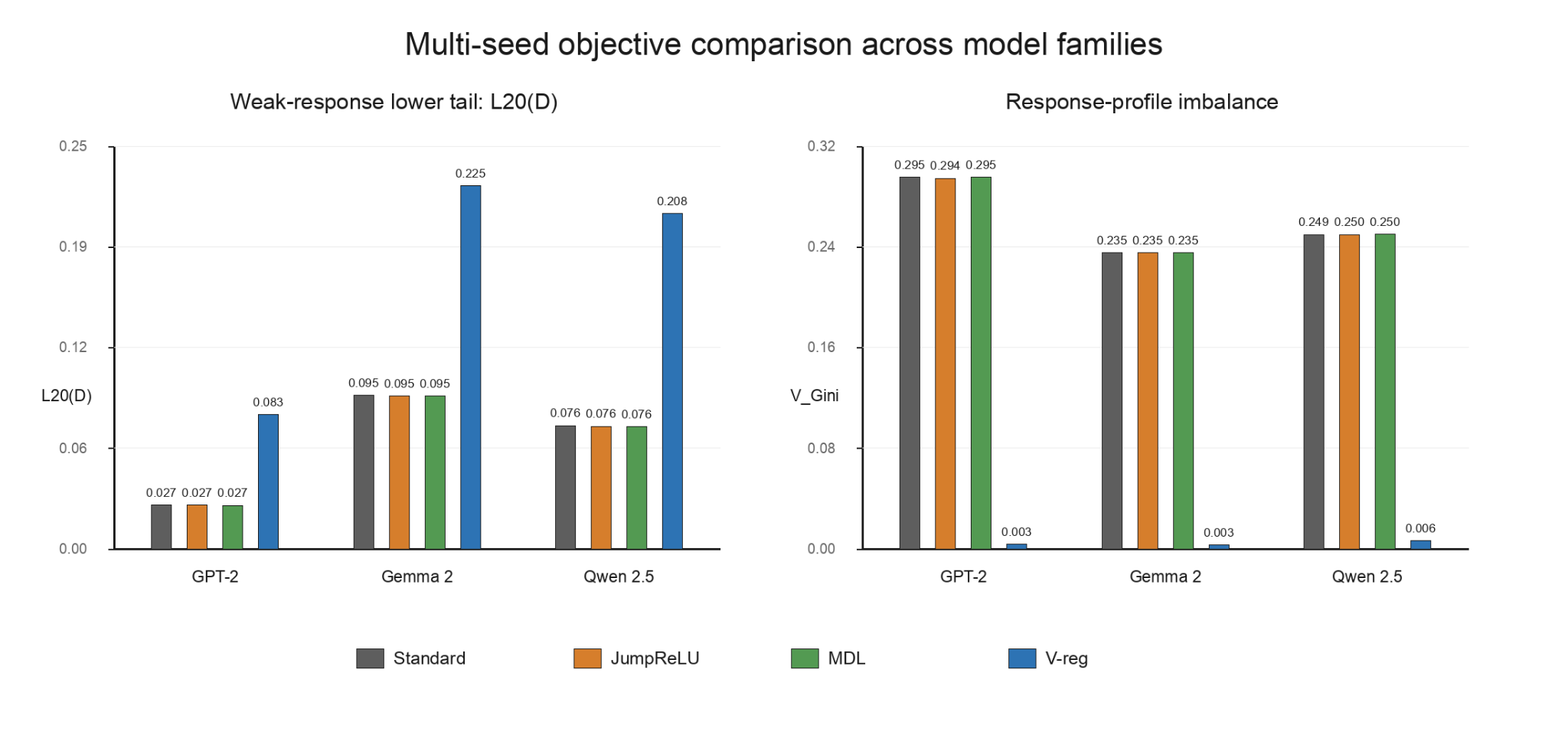}
\caption{Frozen multi-seed baseline extension across Standard, JumpReLU,
MDL, and $V$-regularised SAEs. The non-perturbation-aware objectives stay
close to Standard on lower-tail response and response-profile imbalance,
whereas $V$-regularisation separates on both axes across GPT-2, Gemma~2,
and Qwen~2.5.}
\label{fig:baseline-separation}
\end{figure}

\begin{table}[t]
\centering
\caption{Frozen multi-seed baseline extension. Deltas are each objective
minus the Standard SAE, mean $\pm$ standard deviation across three seeds,
under the same protocol and caches as Table~\ref{tab:e7}. JumpReLU and MDL
stay near Standard, while $V$-regularisation separates on every model.}
\label{tab:e7b}
\begin{center}
\small
\begin{tabular}{@{}llcccc@{}}
\toprule
Model & Objective & $\Delta L_{20}(D)$ & $\Delta L_{20}(\|\Delta z\|)$ &
$\Delta V_{\mathrm{Gini}}$ & Positive families \\
\midrule
\multirow{3}{*}{GPT-2}
& JumpReLU & $-0.0001 \pm 0.0002$ & $+0.0498 \pm 0.0026$ & $-0.0009 \pm 0.0002$ & $0.354 \pm 0.308$ \\
& MDL      & $-0.0003 \pm 0.0002$ & $+0.0315 \pm 0.0040$ & $+0.0001 \pm 0.0002$ & $0.063 \pm 0.108$ \\
& V-reg    & $+0.0562 \pm 0.0010$ & $+5.3584 \pm 0.1202$ & $-0.2919 \pm 0.0008$ & $1.000 \pm 0.000$ \\
\midrule
\multirow{3}{*}{Gemma~2}
& JumpReLU & $-0.0005 \pm 0.0001$ & $+0.1767 \pm 0.0063$ & $-0.0001 \pm 0.0001$ & $0.042 \pm 0.036$ \\
& MDL      & $-0.0006 \pm 0.0001$ & $+0.1084 \pm 0.0011$ & $-0.0001 \pm 0.0000$ & $0.042 \pm 0.072$ \\
& V-reg    & $+0.1301 \pm 0.0042$ & $+10.3454 \pm 1.1553$ & $-0.2322 \pm 0.0001$ & $1.000 \pm 0.000$ \\
\midrule
\multirow{3}{*}{Qwen~2.5}
& JumpReLU & $-0.0005 \pm 0.0001$ & $+0.1539 \pm 0.0072$ & $+0.0001 \pm 0.0001$ & $0.063 \pm 0.063$ \\
& MDL      & $-0.0007 \pm 0.0000$ & $+0.0926 \pm 0.0157$ & $+0.0003 \pm 0.0002$ & $0.042 \pm 0.036$ \\
& V-reg    & $+0.1316 \pm 0.0006$ & $+2.4650 \pm 0.1129$ & $-0.2432 \pm 0.0001$ & $1.000 \pm 0.000$ \\
\bottomrule
\end{tabular}
\end{center}
\end{table}

\subsection{Consistency across three architectures}
\label{sec:res-arch}

The three models differ in parameter count (124M to 3B), training data
(English-only to multilingual), and architecture family (OpenAI, Google,
Alibaba). The frozen multi-seed effect of \S\ref{sec:res-e7} holds for all
three, so the response-profile conditioning is not tied to a single scale,
corpus, or architecture family. Within the single-run view, Qwen~2.5
reaches the lowest post-regularisation imbalance, which motivates the
scale analysis next.

\subsection{Scale-dependence of the numeric response profile}
\label{sec:res-scale}

The number swap family is the most safety-relevant of the six, for the
clinical reason set out in \S\ref{sec:intro}. On GPT-2 (124M), the
$V$-regularised SAE reduces the number swap imbalance from 0.264 to 0.034,
on Gemma~2 (2B) from 0.287 to 0.004, and on Qwen~2.5 (3B) the
post-regularisation value is 0.005. On Gemma~2, the mean number-swap
sensitivity is essentially unchanged (0.073 to 0.075) while the lower tail
roughly doubles (0.035 to 0.074), so the reduction reflects lifted weak
directions rather than uniform suppression. These descriptive values are
consistent with the broader view that scaling can affect the information
available to the SAE, while the SAE objective affects how that information
is arranged in the feature basis.

\subsection{A single orientation for all families}
\label{sec:res-joint}

We trained single-family SAEs (each regularised on one family only) and
evaluated them on all six families. Every single-family SAE reduces
imbalance on its target but leaves the other families largely unchanged. A
joint SAE, regularised on all six families simultaneously, reaches a
low-imbalance regime across the six-family set
(Table~\ref{tab:joint}; the Qwen column uses the method-comparison
V-reg result, which regularises all six families at every step). The
minimum number of parallel feature bases needed for simultaneous
$V \approx 0$ is $m^* = 1$, meaning the six families do not conflict
algebraically and a single well-oriented basis suffices.

\begin{table}[t]
\centering
\caption{Joint $V$-regularised SAE ($\lambda_2 = 0.1$, all six families
simultaneously, 5{,}000 steps) on three models. Raw reconstruction MSE in
parentheses. Values are descriptive single-run point estimates; entries
printed as 0.000 are rounded to three decimals
(Appendix~\ref{app:nearzero}), and multi-seed variability for the primary
lower-tail endpoints is reported in Table~\ref{tab:e7}.}
\label{tab:joint}
\begin{center}
\begin{tabular}{@{}lccc@{}}
\toprule
Family & GPT-2 & Gemma 2 2B & Qwen 2.5 3B \\
\midrule
Negation      & 0.000 & 0.000 & 0.000 \\
Typo          & 0.017 & 0.000 & 0.000 \\
Synonym       & 0.019 & 0.000 & 0.002 \\
Number swap   & 0.036 & 0.000 & 0.005 \\
Semantic sub. & 0.048 & 0.000 & 0.001 \\
Word order    & 0.046 & 0.002 & 0.000 \\
\midrule
Mean          & 0.028 & 0.0003 & 0.001 \\
MSE           & (0.005) & (0.005) & (0.010) \\
\bottomrule
\end{tabular}
\end{center}
\end{table}

In the check-digit framework \citep{b_general}, this is analogous to a
permutation parameter that simultaneously detects substitutions,
transpositions, and twin errors with a single fold of the dihedral group
$D_n^m$. The value $m^* = 1$ means that the six perturbation families
tested here can coexist in one feature basis without algebraic conflict.
When new families are added, the same protocol determines whether the
enlarged set still admits $m^*=1$ or requires $m^*>1$; in the latter
case the multi-factor architecture (Appendix~\ref{app:multifactor})
provides $m$ parallel encoders at constant total capacity. The stabiliser monotonicity theorem of \citet{b_general} proves that
adding factors can only shrink the stabiliser in the algebraic setting, so
$m^*$ quantifies family compatibility for the chosen perturbation set.

\subsection{Sensitivity increase, not suppression}
\label{sec:res-sens}

A natural concern is that low $V$ might be achieved by suppressing all
sensitivity (making every pair equally invisible). The per-family control
in Appendix~\ref{app:sensitivity-control} shows the opposite. On all three
models, the mean sensitivity \emph{increases} for most families under
$V$-regularisation. On GPT-2, negation sensitivity rises by 58\% and word
order by 34\%. On Gemma~2, typo sensitivity more than triples (+219\%)
and number swap more than doubles (+154\%). On Qwen~2.5, every family
shows increased mean sensitivity. These increases show that the
profile becomes less uneven through lifted weak directions rather than
through suppression.

\subsection{Input-normalised gain controls for perturbation complexity}
\label{sec:res-gain}

The six perturbation families differ in complexity (a typo changes one
character; a semantic substitution changes a content word with variable
embedding distance). To verify that the $V$ reduction is not an artefact
of within-family complexity homogeneity, we also compute the
input-normalised gain metric (\S\ref{sec:measure}), which divides the
feature-space sensitivity by the hidden-state perturbation size. The
$V$-regularised SAE with gain normalisation (V-reg$_{\mathrm{n}}$)
achieves $V_{\mathrm{gain}}$ values that round to 0.000 on all six
families on all three models (Appendix~\ref{app:gain}), indicating very
low gain imbalance after input-complexity normalisation. This supports the
interpretation that the effect is not merely complexity homogenisation.

\subsection{Generalisation to sixteen perturbation families}
\label{sec:res-joint16}

To test whether the result is specific to the six families used above,
we extended the joint $V$-regularised training to sixteen families, the
six standard families plus ten auto-generated medical-domain families
(\S\ref{sec:setup}). All sixteen families enter the $V$ loss at every
step (script \texttt{train\_joint.py}, $\lambda_2=0.2$, 7{,}500 steps;
evaluation uses 50 pairs for each of the six original families and 30 for
each of the ten medical families).

Table~\ref{tab:joint16} reports a secondary low-imbalance summary for the
sixteen-family joint runs. On GPT-2 and Qwen, the $V$-regularisation
drives the mean $V$ across all sixteen families to a low-imbalance regime.
GPT-2 falls from 0.314 to 0.002 ($-99.3\%$, with 16/16 families below
0.05 and 15/16 below 0.005), and Qwen falls from 0.239 to 0.000
($-100\%$, with 16/16 families below 0.005).

The extension is informative because the \emph{ten additional
medical-domain families}, which were never used to design the metric or
the training protocol, are reduced at least as thoroughly as the six
original families.
On GPT-2 the ten medical families fall from a mean $V$ of 0.322 to 0.000,
marginally better than the six original families (0.302 $\to$ 0.006), and
on Qwen both groups reach 0.000 (from 0.251 and 0.218 respectively).
Clinically salient families such as \texttt{severity\_change} (GPT-2
$0.428 \to 0.000$), \texttt{frequency\_change} ($0.377 \to 0.000$),
\texttt{drug\_name\_swap} ($0.266 \to 0.000$), and
\texttt{unit\_of\_measure} ($0.296 \to 0.000$) reach low imbalance under
the joint run. Thus, the orientation generalises across perturbation type
rather than only across the families used to derive it, which is the
property a single well-oriented basis would be expected to show. The full
per-family values are reported in Appendix~\ref{app:joint16-detail}.

\begin{table}[t]
\centering
\caption{Secondary low-imbalance summary for sixteen-family joint
$V$-regularised SAE runs. ``Std SAE'' and ``V-reg'' are the mean
$V_{\mathrm{Gini}}$ over all sixteen families; ``$V{<}0.05$'' and
``$V{<}0.005$'' are descriptive low-imbalance counts. This summary covers
the two models for which the sixteen-family joint run was conducted;
Gemma~2 is evaluated under the frozen multi-seed protocol in
Table~\ref{tab:e7}. Source files
\texttt{results/joint\_eval\_\{gpt2,qwen-2.5-3b\}.json}.}
\label{tab:joint16}
\begin{center}
\begin{tabular}{@{}lcccccc@{}}
\toprule
Model & Layer & Std SAE & V-reg & Reduction & $V{<}0.05$ & $V{<}0.005$ \\
\midrule
GPT-2 small & 12 & 0.314 & \textbf{0.002} & $-99.3\%$ & 16/16 & 15/16 \\
Qwen 2.5 3B & 18 & 0.239 & \textbf{0.000} & $-100\%$  & 16/16 & 16/16 \\
\bottomrule
\end{tabular}
\end{center}
\end{table}

%=======================================================================
\section{Discussion}
\label{sec:discussion}

The central finding is that an SAE feature basis can be reoriented at
fixed capacity to redistribute perturbation sensitivity. The five variants
tested here share the same width and training data and differ only in the
training objective. Objectives that optimise reconstruction, sparsity, or
description length leave the perturbation-response profile close to the
Standard SAE. $V$-regularisation changes that profile directly.

This is the same orientation problem that motivates the checksum analogy
in the introduction. A representation can contain enough information for a
distinction to matter and still arrange its feature basis so that some
perturbation directions are weakly expressed. The multi-seed result shows
that the weak lower tail can be lifted across GPT-2, Gemma~2, and
Qwen~2.5, while the sixteen-family runs show that the same orientation can
extend to additional medical-domain perturbation families included in the
sixteen-family joint run. The family-geometry controls rule out the simplest explanations
based on repeated templates or one dominant family.

For AI development practice, the method adds a concrete
perturbation-awareness evaluation and training axis to the usual
reconstruction and sparsity criteria. A team
building a clinical-facing language model could specify safety-relevant
perturbation families such as dosage changes, negation of
contraindications, or unit-of-measure substitutions, measure the $V$
profile of the feature decomposition, and apply $V$-regularisation to
families with high imbalance. The intervention modifies only the SAE
training loss, not the base model or the deployment stack.

The practical stakes are broad. The WHO projects a global shortage of
11 million health workers by 2030 \citep{b_who}, and clinical workflows
may increasingly depend on AI-assisted pipelines with limited human
oversight. In such settings, a corrupted or adversarially altered input,
for example ``80~mg'' instead of ``8~mg'', can propagate if the internal
representation treats the two values too similarly. A $V$ profile, interpreted alongside
lower-tail and absolute-response quantities, gives developers a way to
locate perturbation families and model layers where weak responses are
concentrated. It can guide targeted SAE
regularisation, identify perturbation families for additional evaluation,
and inform perturbation-sensitive SAE retraining for the affected family.
In that role, the hierarchy contributes to safer AI
deployment by making
sensitivity gaps measurable and actionable, while downstream task
validation remains the final test of application behaviour.

The held-out, calibrated, and external probes place this representation
result in context. They show that improving the sparse-code response
profile can align with downstream separation, especially for Gemma~2 and
Qwen~2.5 under the null-calibrated code-distance test, but they also show
that probability-margin transfer is task-specific. This is why
perturbation awareness is evaluated through the perturbation-response
hierarchy in \S\ref{sec:measure}, summarised in
Table~\ref{tab:metric_hierarchy}. $V_{\mathrm{Gini}}$ controls
the shape of the response profile, lower-tail and absolute-response
metrics check that weak directions have been lifted, and downstream probes
measure task accessibility. This interpretation rests on a precondition.
$V$-regularisation reorients the feature basis but cannot create
information that the hidden state does not encode. When the relevant
distinction is not represented at the chosen layer, no change to the SAE
objective can recover it, and the limitation lies in the hidden state
rather than in the regulariser.

\paragraph{Sensitivity to the choice of perturbation family.}
The conclusions depend on the perturbation families the practitioner
specifies, and this dependence has two sides. On the training side, the
sixteen-family study and the family-geometry controls show that the
low-imbalance regime is not driven by a single dominant family or by
template duplication. On the evaluation side, the held-out family audit
(Appendix~\ref{app:unseen}) evaluates the frozen checkpoints on a
pre-constructed pool of one hundred families, of which eighty-four were
never used in training, and finds the lower-tail lift at the same order of
magnitude on unseen instances of seen kinds and on eighteen structurally
novel kinds. The lift is therefore a property of the learned basis rather
than of the specific family list, while the downstream ceiling cases show
that the choice of family still determines whether the lift has room to
translate into separability. Specifying families remains a modelling
decision, and the framework treats it as an explicit input rather than a
hidden assumption.

\paragraph{Scope and negative results.}
The effect is scoped to endpoints and models with headroom in the Standard
code. On a numeric dose-value
probe the surface token is linearly decodable and the Standard code is at
ceiling, so $V$-regularisation is slightly worse there, with a
$\Delta$AUROC between $-0.04$ and $-0.16$ across the three models
(Appendix~\ref{app:dosage}). At the
representation level the numeric family shows the same weak lower tail as
the other families and the objective lifts it, yet at the readout level
the dose value remains at ceiling because the digit is surface-decodable,
which is why a value probe is not the informative downstream test and why
the safety concern applies where this surface redundancy is absent. On
real OpenI negation under report-grouped cross-validation the Standard
code is only marginally sub-ceiling and the $V$-regularised code does not
reliably improve it, with a $\Delta$AUROC interval of
$[-0.002, +0.008]$ that includes zero, a null result. Where
the Standard code already separates a family at ceiling, a focused
single-family objective can reduce held-out linear separability rather
than improve it, from an AUROC of $1.00$ to $0.81$ on the dedicated Qwen
negation run. On Qwen the real-clinical lower-tail lift is accompanied by
a higher reconstruction error in the perturbation-pair domain, with
per-family normalised mean squared error rising from about
$0.006$--$0.008$ to about $0.040$--$0.044$, while general OpenWebText
reconstruction is unchanged. These cases share one explanation. With
frozen general checkpoints, tasks that are saturated in the Standard code
or carried by a trivially decodable surface token leave no accessibility
gap to close. The gains we report are therefore specific to endpoints with
demonstrated headroom, and $V_{\mathrm{Gini}}$ is not a standalone
certificate of downstream separability. Two preflight checks give the same
reading before any focused training. A numeric orientation stress test
finds no clean non-saturated regime on Qwen~2.5
(Appendix~\ref{app:stress}), and a severity headroom preflight finds the
severity family already saturated in the Standard code on all three models
(Appendix~\ref{app:preflight}). These boundary cases delimit where the
method helps, and they leave the broader reading unchanged. The
perturbation-response hierarchy introduced in \S\ref{sec:measure} makes
representational blind spots formally definable and their sampled
operational signatures measurable, while the differentiable objective
conditions its finite response-profile imbalance component. The gains
concentrate where the representation has room to improve.

Beyond this empirical scope, the framework connects naturally to cognitive
science. Human
perceptual blind spots, including change blindness and inattentional
blindness \citep{b_simons}, can be interpreted through the same algebraic
lens. Input changes that fall in the stabiliser of the current attentional
state remain undetected. Looking forward, SAE training should be viewed as
a multi-objective problem over reconstruction, sparsity, perturbation
awareness, and compressibility. More broadly, turning representational
blind-spot analysis from an informal intuition into a computable hierarchy
with a differentiable imbalance component is a precondition for
automation. An algorithm cannot search for
blind spots that are not first defined as a computable quantity. By
specifying response-profile imbalance as a differentiable training signal,
this work makes perturbation awareness actionable in SAE training. Future
methods could in principle discover and audit relevant weak directions
without relying on a perturbation family supplied by hand, as is done
here.
Mapping the surface $V(\lambda_2, d_{\mathrm{sae}}, \mathrm{scale})$ and
determining the orientation cost $m^*$ for broader perturbation taxonomies
are natural next steps.

%=======================================================================
\section{Conclusion}
\label{sec:conclusion}

Representational blind spots in language models can be reduced by
conditioning the SAE feature basis, without retraining the base model. A
single differentiable regularisation term, optimising the Gini coefficient
of perturbation responses, reduces response-profile imbalance and lifts
weak perturbation directions at fixed SAE width. Across GPT-2, Gemma~2,
and Qwen~2.5, the frozen multi-seed protocol shows consistent lower-tail
lifting and imbalance reduction. On GPT-2 and Qwen~2.5, the sixteen-family
joint runs drive the mean $V$ across additional medical-domain perturbation
families to a low-imbalance regime. The result is not that one metric
settles downstream safety, but that response-profile conditioning provides
an actionable route toward perturbation awareness.

%=======================================================================
\subsubsection*{Broader Impact Statement}

This work targets a safety-relevant failure mode of language models used
in clinical and other high-stakes settings. The relevant failure is weak
or uneven representational separation for input changes that matter for
the task, such as numeric dosages, negations, or unit-of-measure
substitutions. The $V$ measure and its regularisation are intended to make
sensitivity imbalance measurable and reducible, and we expect the primary
impact to be positive. The method can help developers inspect and reduce
response-profile vulnerabilities before deployment. The $V$ profile should
complement downstream clinical validation and human oversight. In
particular, $V_{\mathrm{Gini}}$ measures within-family response uniformity
rather than absolute task detectability. A system that is uniformly
insensitive across a safety-relevant family could still have low
$V_{\mathrm{Gini}}$, so any certification use would require
absolute-sensitivity thresholds or downstream validation in addition to
the Gini profile. The perturbation families used here are
synthetic and English-language; deployment in other languages or clinical
sub-domains requires deriving family-specific perturbation sets for that
setting.

%=======================================================================
\phantomsection
\section*{Reproducibility and supplementary material}
\label{sec:repro}

This paper is self-contained. The main text states all settings needed to
reproduce the main results, and Appendices~\ref{app:diff}--\ref{app:families}
give the full derivations, family-geometry controls, reporting
conventions, protocol provenance, per-family and per-axis tables, and
example sentence pairs for all sixteen perturbation families.

As supplementary material, we provide the datasets and splits
used for the perturbation probes, the machine-readable result files behind
the reported tables and figures, the protocol manifests for the frozen
multi-seed experiments, and the code used for cache building, training,
evaluation, template-cluster resampling, and null-calibrated analyses. The
supplementary archive also includes the OpenI-derived minimal-pair splits
used for the external checks and a manifest mapping result files to the
tables and figures in the paper. The pipeline uses PyTorch and the
HuggingFace \texttt{transformers} library and reproduces on a single
consumer GPU. Activation caches, feature caches, and SAE checkpoints are
regenerated by the scripts on first run and are therefore not bundled.
The code, frozen evaluation data, machine-readable results, and provenance
manifest are archived at
\url{https://github.com/csabibme/blind_spot_reduction_SAE_TMLR}, with a
fixed camera-ready snapshot as release \texttt{v1.0.0}.

\section*{AI Declaration}

During the preparation of this work Anthropic Claude was used to assist
in refining the linguistic and mathematical prose and in drafting the
companion training and evaluation scripts. The conceptual framework, the
results, and the experimental programme are the authors' own. All content
was reviewed and edited by the authors, who take full responsibility for
the published article.

%=======================================================================
\bibliography{main}
\bibliographystyle{tmlr}

%=======================================================================
\appendix

\section{Differentiability of \texorpdfstring{$V_{\mathrm{Gini}}$}{V-Gini}}
\label{app:diff}

Section~\ref{sec:measure-gini} derives the sorted form Eq.~\eqref{eq:gini}
from the pairwise form Eq.~\eqref{eq:gini_pairwise} through the
one-sided signed-rank coefficient $2i-(K+1)$ and the factor of two from
double-counting unordered pairs. We record here the boundary behaviour the derivation assumes. The
analytical $V_{\mathrm{Gini}}$ is defined for every nonnegative profile
with $S=\sum_i D_i>0$. At the zero profile we assign
$V_{\mathrm{Gini}}=0$ by convention; because the metric is
scale-invariant, this assignment is not a continuous extension. For fixed
$K$ the analytical metric lies in $[0,(K-1)/K]$: zero is attained on every
nonzero equal profile, and the upper endpoint is attained when all response
mass is concentrated on a single transformation. Because the responses
are nonnegative by construction, no sign correction is needed.

On every region with a fixed strict ordering of the responses,
$V_{\mathrm{Gini}}$ is a smooth linear-fractional function of the
unsorted responses $D_1,\ldots,D_K$. At ties the classical gradient need
not be unique, although ties do not change the value of the statistic.
Our implementation uses \texttt{torch.sort}; its backward pass propagates
gradients through the returned permutation and thereby selects one
order-consistent branch.

For the analytical metric, on a strict-ordering region with $S>0$, the
gradient with respect to the unsorted response $D_j$ is
\begin{equation}
\frac{\partial V_{\mathrm{Gini}}}{\partial D_j} =
  \frac{2\,\pi(j)}{K \cdot S}
  - \frac{2\sum_{i=1}^K i\,D_{(i)}}{K \cdot S^2},
\end{equation}
where $\pi(j)$ is the rank of $D_j$. The numerical training implementation
uses $S+10^{-8}$ in the denominator and clamps the result to $[0,1]$; it is
therefore an $\varepsilon$-stabilised surrogate of the analytical metric
at and near the zero profile. The separate $\varepsilon$ in
Eq.~\eqref{eq:response} stabilises each response denominator but does not
guarantee $S>0$. Away from sorting ties and ReLU boundaries, gradients
propagate through the encoder by the chain rule; at those boundaries
PyTorch applies its standard branch/subgradient conventions.

\section{Exact worst-subset interpretation of lower-tail response}
\label{app:lower-tail-guarantee}

The independently selected lower-tail endpoint has an exact finite-family
interpretation. Let $I=\{1,\ldots,n\}$ index the same equally weighted
finite family of evaluated pairs under the Standard and $V$-regularised
SAEs, and let $m=\lceil0.20n\rceil$.

\begin{proposition}[Worst-subset characterisation of the lower tail]
\label{prop:worst-subset}
For a response profile $D=(D_1,\ldots,D_n)$ with ordered entries
$D_{(1)}\leq\cdots\leq D_{(n)}$,
\begin{equation}
L_{20}(D)
=
\frac{1}{m}\sum_{j=1}^{m}D_{(j)}
=
\min_{\substack{W\subseteq I,\; |W|=m}}
\frac{1}{m}\sum_{i\in W}D_i.
\label{eq:worst-subset}
\end{equation}
\end{proposition}

\begin{proof}
Among all subsets of cardinality $m$, the smallest possible sum is
obtained by selecting the $m$ smallest entries of $D$; division by $m$
gives the result. With ties the minimiser need not be unique, but the
minimum value is unchanged.
\end{proof}

\begin{corollary}[Standard-defined weak-set guarantee]
\label{cor:standard-weak}
Let $W_S$ be any size-$m$ subset attaining the Standard lower-tail
minimum. Then
\[
\overline D_S(W_S)=L_{20}(D_S),
\qquad
\overline D_V(W_S)\geq L_{20}(D_V),
\]
and consequently
\begin{equation}
\overline D_V(W_S)-\overline D_S(W_S)
\geq
L_{20}(D_V)-L_{20}(D_S).
\label{eq:standard-weak-guarantee}
\end{equation}
\end{corollary}

Thus, whenever the independently computed relative lower-tail endpoint
increases, the mean relative response on a Standard-defined minimising
weak set must increase by at least the reported own-tail lift. This is a
subset-average guarantee, not a pointwise one: it does not imply
$D_i^V>D_i^S$ for every $i\in W_S$. It also concerns point estimates on
the same finite universe; uncertainty intervals for the own-tail and
fixed-set estimands retain their stated resampling procedures.

\paragraph{Relation to the normalised response profile.}
For a nonzero response profile, let $p_i=D_i/\sum_jD_j$ and let
$\Lambda_k=\sum_{i=1}^k p_{(i)}$ denote the cumulative share of the $k$
smallest normalised responses (distinct from the coverage statistic
$C_\gamma$). The standard discrete Lorenz--Gini identity
\citep{cowell2011} gives
\begin{equation}
V_{\mathrm{Gini}}
=
\frac{n-1}{n}
-
\frac{2}{n}\sum_{k=1}^{n-1}\Lambda_k.
\label{eq:lorenz-gini}
\end{equation}
Hence a reduction in $V_{\mathrm{Gini}}$ is exactly an increase in the
aggregate cumulative lower-tail share of the normalised response
profile. This aggregate identity does not imply pointwise Lorenz
dominance or identify which perturbation pair carries a response.

\paragraph{Standard-defined OpenI weak sets.}
Table~\ref{tab:openi-fixed-weak} evaluates both SAEs on the same
Standard-defined weak pairs. The relative point estimates numerically
verify Corollary~\ref{cor:standard-weak} in all eight cells and are not
independent evidence beyond the own-tail endpoint. The actual values,
pair-level fractions, and conditional intervals are reported for
transparency. The absolute $\|\Delta z\|$ endpoint on the relative-$D$
weak set is not implied by the corollary and remains separate empirical
evidence.

\begin{table*}[t]
\centering
\caption{Standard-defined OpenI weak-set evaluation. $m=|W_S|$;
$\Delta L_{20}(D)$ uses independently selected own tails, while
$\Delta\overline D(W_S)$ evaluates the fixed Standard weak set.
Intervals are conditional paired 95\% bootstrap intervals on the fixed
set. ``$D_i\uparrow$'' is the fraction of its pairs with larger relative
response under $V$-regularisation. The final column is the absolute
code-displacement change on the same relative-$D$ weak set.}
\label{tab:openi-fixed-weak}
\scriptsize
\setlength{\tabcolsep}{3pt}
\begin{tabular}{@{}llrrrcr@{}}
\toprule
Model & Family & $m$ & $\Delta L_{20}(D)$ &
$\Delta\overline D(W_S)$ [95\% CI] & $D_i\uparrow$ &
$\Delta\overline{\|\Delta z\|}(W_S)$ [95\% CI] \\
\midrule
GPT-2 & negation & 60 & $+0.004215$ & $+0.004736$ [$+0.003720,+0.005731$] & 0.917 & $+0.6337$ [$+0.5151,+0.7537$] \\
GPT-2 & laterality & 60 & $+0.000945$ & $+0.000995$ [$+0.000886,+0.001104$] & 0.983 & $+0.1202$ [$+0.1053,+0.1348$] \\
GPT-2 & severity & 60 & $+0.003231$ & $+0.004157$ [$+0.003558,+0.004791$] & 1.000 & $+0.5350$ [$+0.4577,+0.6183$] \\
GPT-2 & anat.\ dir. & 27 & $+0.000807$ & $+0.000949$ [$+0.000771,+0.001149$] & 1.000 & $+0.1161$ [$+0.0954,+0.1388$] \\
\midrule
Qwen~2.5 & negation & 16 & $+0.009887$ & $+0.009887$ [$+0.008415,+0.011205$] & 1.000 & $-0.1080$ [$-0.1963,-0.0271$] \\
Qwen~2.5 & laterality & 16 & $+0.000888$ & $+0.000960$ [$+0.000769,+0.001181$] & 1.000 & $-0.0066$ [$-0.0132,-0.0008$] \\
Qwen~2.5 & severity & 16 & $+0.003506$ & $+0.003601$ [$+0.002771,+0.004417$] & 1.000 & $-0.0029$ [$-0.0295,+0.0266$] \\
Qwen~2.5 & anat.\ dir. & 16 & $+0.000556$ & $+0.000562$ [$+0.000408,+0.000705$] & 0.875 & $-0.0015$ [$-0.0051,+0.0021$] \\
\bottomrule
\end{tabular}
\end{table*}

The absolute endpoint is positive with an interval excluding zero in all
four GPT-2 cells. On Qwen it is negative with an interval excluding zero
for negation and laterality, and unresolved for severity and anatomical
direction. We therefore describe the OpenI headline as a relative
lower-tail lift and do not claim a general absolute repair.

\paragraph{Cross-fitted Qwen semantic accessibility.}
The Qwen analysis uses all 78 frozen negation pairs
in a nested five-fold cross-fit, keeps template groups intact, and
selects the 16 weak pairs once by Standard $D_i$.
Table~\ref{tab:qwen-fixed-probe} reports both that set and the remaining
62 pairs. Low Standard $D_i$ identifies a lower Standard probability
margin: the weak-minus-non-weak difference is $-0.287$
[$-0.509,-0.066$]. The error endpoint does not show the same ordering,
reflecting its near-ceiling discreteness. Because Standard $D_i$ and the
Standard margin share a representation source, this association is
descriptive rather than an independent causal test.

\begin{table}[t]
\centering
\caption{Cross-fitted Qwen negation probe on the fixed Standard-defined
weak set and remaining pairs. Margin intervals are template-clustered
95\% intervals for the paired $V$-regularised-minus-Standard change.}
\label{tab:qwen-fixed-probe}
\small
\begin{tabular}{@{}lrrrrr@{}}
\toprule
Subset & $n$ & Standard margin & V-reg margin & $\Delta$ margin [95\% CI] &
Error S$\to$V \\
\midrule
Standard weak 20\% & 16 & 0.302 & 0.358 & $+0.055$ [$-0.087,+0.257$] & $1/16\to0/16$ \\
Remaining 80\% & 62 & 0.590 & 0.575 & $-0.014$ [$-0.123,+0.118$] & $7/62\to2/62$ \\
\bottomrule
\end{tabular}
\end{table}

On the weak set the single discordant correctness transition is
$V$-reg-only correct and the exact two-sided McNemar value is $p=1.00$.
On the remaining set the corresponding value is $p=0.180$. Thus low
Standard response is associated with lower continuous semantic margin,
but the Standard-to-$V$-regularised margin change on those same weak
pairs remains statistically unresolved. Separate representation-specific
probes also make this a representation-plus-readout comparison rather
than a pure geometric intervention.

\section{Real-clinical lower-tail audit and qualitative case study}
\label{app:casestudy}

This appendix supports the real-clinical audit of
\S\ref{sec:res-e3} with the per-family table and a qualitative view of
what a weak response looks like on concrete minimal pairs. The evaluation
uses the frozen sixteen-family general checkpoints, zero-shot, on OpenI
chest radiography reports.

\paragraph{What a blind spot looks like.}
In a sampled evaluation, a clinical blind spot would appear as a
task-relevant minimal pair for which the SAE code barely moves. Consider a
negation pair and a laterality pair drawn from the OpenI reports:
\begin{quote}
\emph{Affirmed.} ``There is a small pleural effusion.''\\
\emph{Negated.} ``There is no pleural effusion.''
\end{quote}
\begin{quote}
\emph{Left.} ``Opacity in the left lower lobe.''\\
\emph{Right.} ``Opacity in the right lower lobe.''
\end{quote}
Under the Standard SAE code, the weakest responses in these families sit
low. The lower-tail response $L_{20}(D)$ is $0.030$ for negation and
$0.0034$ for laterality on Qwen~2.5, and $0.016$ and $0.0024$ on GPT-2.
These weak-tail values are consistent with the code-space signature of a
blind spot, meaning that for the hardest pairs in each family, the code
moves little when the clinical meaning flips.

\paragraph{How the regulariser changes it.}
On the same frozen checkpoints the $V$-regularised code raises the weak
tail in every family and on both models
(Table~\ref{tab:rclt}). The lower tail rises in all eight model-by-family
cells, and the report-clustered interval excludes zero in each. The
largest lift is on negation for Qwen~2.5, where $L_{20}(D)$ moves from
$0.030$ to $0.040$.

\begin{table}[t]
\centering
\caption{Real-clinical lower-tail response on OpenI minimal pairs, frozen
general checkpoints, zero-shot. $L_{20}(D)$ is the mean relative code
displacement over the weakest quintile of pairs. Intervals are
report-clustered 95\% intervals of $\Delta L_{20}$.}
\label{tab:rclt}
\small
\begin{tabular}{@{}llccc@{}}
\toprule
Model & Family & $L_{20}$ Standard & $L_{20}$ V-reg & $\Delta L_{20}$ [95\% CI]\\
\midrule
GPT-2 & negation      & 0.0157 & 0.0200 & $+0.0042$ [$+0.0030,+0.0053$]\\
GPT-2 & laterality    & 0.0024 & 0.0034 & $+0.0009$ [$+0.0008,+0.0011$]\\
GPT-2 & severity      & 0.0044 & 0.0077 & $+0.0032$ [$+0.0027,+0.0038$]\\
GPT-2 & anat.\ dir.   & 0.0017 & 0.0025 & $+0.0008$ [$+0.0007,+0.0010$]\\
\midrule
Qwen~2.5 & negation   & 0.0298 & 0.0397 & $+0.0099$ [$+0.0081,+0.0123$]\\
Qwen~2.5 & laterality & 0.0034 & 0.0043 & $+0.0009$ [$+0.0007,+0.0012$]\\
Qwen~2.5 & severity   & 0.0117 & 0.0153 & $+0.0035$ [$+0.0026,+0.0046$]\\
Qwen~2.5 & anat.\ dir.& 0.0021 & 0.0027 & $+0.0006$ [$+0.0004,+0.0008$]\\
\bottomrule
\end{tabular}
\end{table}

\paragraph{The lift is not perfectly rank-aligned with the trained statistic.}
The lower-tail response and the trained inequality statistic are not
perfectly rank-aligned on this data. On Qwen~2.5 the largest lower-tail lift,
negation at $+0.0099$, accompanies the smallest change in
$V_{\mathrm{Gini}}$, from $0.427$ to $0.424$. On GPT-2 the largest drop in
$V_{\mathrm{Gini}}$, anatomical direction from $0.644$ to $0.588$,
accompanies the smallest lower-tail lift, $+0.0008$. The rank order of the
lower-tail lift does not follow the rank order of the inequality change,
so the observed magnitude change is not a restatement of the rank order
of the shape statistic. The scale-invariant objective does not itself
reward a larger mean or absolute response, although the coupled SAE
parameter update can change both indirectly.

\paragraph{Scope.}
The analysis stays at the level of the response profile, since no
feature-activation study is included, so it supports no claim about
monosemanticity of individual units or about a universal downstream gain. On
Qwen~2.5 the lower-tail lift in this domain is accompanied by a higher
perturbation-pair reconstruction error, reported in the scope paragraph of
\S\ref{sec:discussion}. The case study establishes the narrower and
verifiable point that the Standard-defined weak sets receive larger mean
relative responses under $V$-regularisation on real clinical text. This
set-level result does not imply that every individual weak pair improves,
and the absolute endpoint is mixed on Qwen as reported in
Table~\ref{tab:openi-fixed-weak}.

%=======================================================================
\section{Dedicated paired experiment}
\label{app:dedicated}

The dedicated experiment isolates the effect of the $V$ term with a
paired common-randomness design. For each model, a Standard and a
$V$-regularised SAE are trained from identical initialisation, an
identical reconstruction minibatch stream, and identical data order,
differing only in the presence of the $V$ term with $\lambda_2=0.2$ on the
negation family. Three outcomes summarise the run. First, the phenomenon
is isolated, since the trained family's response inequality collapses from
$0.233$ to $0.0009$ on GPT-2 and from $0.113$ to $0.0008$ on Qwen~2.5, a
reduction of about $99$ percent in both cases. Second, general
reconstruction is unharmed, since explained variance on held OpenWebText
tokens is unchanged at $0.9997$ on GPT-2 and rises from $0.9950$ to
$0.9958$ on Qwen~2.5. Third, the run exposes the saturation boundary
reported in the scope paragraph of \S\ref{sec:discussion}, since the
Standard Qwen code already separates the negation family at an AUROC of
$1.00$ and the focused single-family objective reduces held-out linear
separability to $0.81$. The dedicated experiment therefore serves as a
phenomenon and no-harm isolation, while the claim-carrying downstream
evidence comes from the endpoints of \S\ref{sec:res-e3}, where the
Standard code has headroom.

%=======================================================================
\section{Controlled non-saturated experiment}
\label{app:toy}

The synthetic experiment constructs a regime in which the task
distinction is present in the input representation but not saturated in
the Standard code, so that headroom exists by design. Inputs are drawn
from a linear generative model in which a designated task direction
carries a controlled fraction of the input variance, governed by a
strength parameter $\alpha$, and paired Standard and $V$-regularised SAEs
are trained with common randomness as in Appendix~\ref{app:dedicated}. A
fixed, rule-based operating point selects the regime before evaluation,
with the rule that the Standard held-out probe sits below an AUROC of
$0.95$ while input-level recoverability stays above $0.99$.

At the selected point the $V$-regularised code gives a held-out probe
AUROC gain of $+0.016$ over the Standard code, with a seed-level standard
error of comparable size, together with a lower-tail displacement gain
$\Delta L_{20}(\|\Delta z\|)$ of $+0.51$ and a reconstruction ratio within
$9$ percent of the Standard run. Sweeping $\alpha$ toward stronger task
signal drives the Standard probe to ceiling and the probe gain to zero,
reproducing in a controlled setting the headroom dependence observed on
the real models. The experiment is reported as a mechanism check rather
than as a headline result, since the probe gain is modest, and it is read
together with the real-model endpoints of \S\ref{sec:res-e3}.

%=======================================================================
\section{Held-out family generalisation audit}
\label{app:unseen}

The audit tests whether the lower-tail lift is specific to the sixteen
training families or a property of the learned basis. The evaluation draws
on a pre-constructed pool of one hundred perturbation families with thirty
fixed pairs each. The pool metadata records a creation date of 25~May
2026, and the held-out status is checkable from the released materials
independently of any timestamp, since the eighty-four held-out families
appear in no training script, cache manifest, or reported training result
of this paper. Sixteen
entries of the pool are the training families themselves,
and the remaining eighty-four are template-generated held-out families.
The frozen general checkpoints are evaluated zero-shot on this held-out
part, comprising two tiers. Tier~A
contains unseen instances of seen family kinds, and Tier~B contains
eighteen family kinds with no mechanical overlap with any training
family, including unit reformatting, temporal shifts, dose-route changes,
and word-order alternations. Families are subsampled to thirty pairs with
per-family fixed seeds, and intervals are pair-level bootstrap intervals.

Six template families are exact duplicates of another pool entry and are
excluded. Seven negation shards in the pool share their weakest pairs, so
their lower-tail statistics coincide, and they are counted as one family
in the macro summaries. After these exclusions, Tier~A comprises
fifty-four families. The lower-tail lift generalises in both tiers. On Tier~A the
macro mean $\Delta L_{20}$ is $+0.016$ on GPT-2, $+0.036$ on Qwen~2.5, and
$+0.026$ on Gemma~2, with the interval excluding zero for $91$, $89$, and
$100$ percent of families respectively, and with every Gemma~2 family
individually positive. On Tier~B, the structurally novel tier, the macro
mean is $+0.026$ on GPT-2, $+0.040$ on Qwen~2.5, and $+0.033$ on Gemma~2,
with every GPT-2 and Gemma~2 family individually interval-positive. The
families that fail to reach interval positivity concentrate where the
scope paragraph of \S\ref{sec:discussion} predicts, on Qwen~2.5
surface-numeric shards. In this out-of-domain perturbation set the
$V$-regularised code shows a higher perturbation-pair reconstruction
error, most pronounced on Qwen~2.5, consistent with the reconstruction
cost reported for the real-clinical domain, while general OpenWebText
reconstruction is unaffected.

%=======================================================================
\section{Dose-value probe and the surface-token ceiling}
\label{app:dosage}

This appendix supports the dose-value discussion in the scope paragraph of
\S\ref{sec:discussion}. A linear probe reads the numeric dose value from
each representation on a held test split. In the hidden state the value is
fully decodable, at an AUROC of $1.00$ on all three models, and the
Standard SAE code preserves this, also at $1.00$. The $V$-regularised code
is slightly worse on the value-distance endpoint, with a test
$\Delta$AUROC of $-0.099$ on GPT-2, $-0.039$ on Gemma~2, and $-0.156$ on
Qwen~2.5. The reading is that the digit sits on the surface of the text
and is linearly decodable regardless of the SAE objective, so the
dose-value distance probe has no headroom and is not the informative
downstream test. The representation-level weak tail on the numeric family
is lifted by the objective in the same way as the other families, as the
held-out family audit shows (Appendix~\ref{app:unseen}); the ceiling is a
property of the readout, not of the code response.

%=======================================================================
\section{Numeric orientation stress test}
\label{app:stress}

The stress test asks whether a non-saturated numeric regime can be found
on the strongest model, so that a focused objective would have room to
act. On Qwen~2.5 we score raw next-token log-probability margins on
numeric minimal pairs, split into an easy regime and a trap regime
designed to require the numeric distinction. The model handles both
without a clean weak regime. Accuracy is $0.58$ on the easy split and
$0.87$ on the trap split, and the lower-tail correct-margin remains near
or above zero across families, so no clean non-saturated numeric regime
is available on this model for a focused intervention. This null informs
the choice, in the controlled experiment of Appendix~\ref{app:toy}, to
construct headroom synthetically rather than search for it in a saturated
numeric task.

%=======================================================================
\section{Severity headroom preflight}
\label{app:preflight}

Before committing a focused severity family, we measured the headroom
available in each model on a severity-change probe with per-seed
resampling. The hidden state already separates severity at an AUROC of
$0.92$ on GPT-2, $1.00$ on Gemma~2, and $0.97$ on Qwen~2.5, and the
Standard SAE code is at or near this level. A focused $V$-regularised code
does not improve on it, with a gate $\Delta$AUROC of $-0.156$ on GPT-2,
$-0.061$ on Gemma~2, and $-0.017$ on Qwen~2.5. The preflight therefore
predicts, before any focused training, that severity is a saturated target
on frozen general checkpoints, which is consistent with the general
finding that gains concentrate where the Standard code has headroom.

%=======================================================================
\section{Certificate-gated laterality repair}
\label{app:repair}

As a proof of mechanism we test whether a targeted repair, applied only
to a certified laterality direction, changes downstream separation on the
external OpenI left/right minimal pairs. On the same report-clustered
protocol as the external boundary check, the zero-shot $V$-regularised
code and a targeted-repair code give small report-cluster deltas with
intervals that include zero, for example a lower-tail probability-margin delta
of $+0.0022$ with interval $[-0.035, +0.030]$ for the zero-shot code on
GPT-2. The repair moves the intended direction without a measurable
downstream gain on this saturated external task, which is consistent with
the boundary reading in \S\ref{sec:res-e3}. The experiment is reported as
a mechanism probe rather than as evidence of downstream improvement.

%=======================================================================
\section{Multi-factor SAE architecture}
\label{app:multifactor}

The multi-factor SAE mirrors the multi-check-digit construction $D_n^m$
from the error-detecting-code framework \citep{b_general}, in which
$m$ independent check functions over an alphabet of size $n$ jointly
detect a wider class of errors than any single function. It consists of
$m$ parallel encoders, each mapping the hidden
state $h \in \mathbb{R}^{d_{\mathrm{in}}}$ to a factor code
$z^{(k)} \in \mathbb{R}^{d_{\mathrm{factor}}}$, and a shared decoder
mapping the concatenation $z = [z^{(1)};\ldots;z^{(m)}]$ back to
$\hat{h}$. Total capacity is held constant at
$d_{\mathrm{factor}} \times m = d_{\mathrm{sae}}$, so the comparison
between $m{=}1$ and $m{>}1$ isolates the effect of factorisation at fixed
capacity. Families are assigned to factors round-robin
($i \bmod m$); the $V$ loss for factor $k$ is computed on $z^{(k)}$ only,
while evaluation computes $V_{\mathrm{Gini}}$ on the concatenated code. In
the six-family descriptive analysis, a single encoder regularised on all
families reaches the low-imbalance regime. The architecture was tested at
$m = 1,2,3,6$ to confirm the $m{=}1$ result is not an artefact of
the specific factorisation. (An orthogonality penalty
$\lambda_{\mathrm{ortho}}\|{W^{(j)}}^\top W^{(k)}\|_F^2$ encourages the
factors to span complementary subspaces.)

\section{Single-run objective profiles}
\label{app:objective-comparison}

The single-run family-level comparison below supports the multi-seed
baseline extension in Table~\ref{tab:e7b}. It is retained as numerical
detail because it reports every evaluated perturbation family and includes
the input-normalised $V$ variant.

\begin{table}[h]
\centering
\caption{$V_{\mathrm{Gini}}$ (raw) across five SAE methods and three
models (5{,}000 steps, $d_{\mathrm{sae}} = 4{,}096$, 50 pairs per
family). Bold indicates the lowest $V$ per family per model. The
standard, JumpReLU, and MDL SAEs produce similar single-run profiles; the
$V$-regularised SAE reduces response-profile imbalance in this
descriptive comparison.}
\label{tab:comparison}
\small
\begin{center}
\begin{tabular}{@{}ll*{5}{c}@{}}
\toprule
Model & Family & Std & Jump & MDL & V-reg & V-reg$_{\mathrm{n}}$ \\
\midrule
\multirow{6}{*}{\rotatebox{90}{GPT-2}}
& Negation      & .219 & .222 & .223 & \textbf{.000} & .267 \\
& Typo          & .307 & .306 & .309 & \textbf{.014} & .295 \\
& Synonym       & .307 & .310 & .310 & \textbf{.018} & .326 \\
& Number swap   & .264 & .265 & .266 & \textbf{.034} & .276 \\
& Semantic sub. & .435 & .433 & .435 & \textbf{.045} & .400 \\
& Word order    & .281 & .282 & .282 & \textbf{.044} & .302 \\
\midrule
\multirow{6}{*}{\rotatebox{90}{Gemma 2}}
& Negation      & .107 & .106 & .111 & \textbf{.000} & .188 \\
& Typo          & .226 & .214 & .219 & \textbf{.000} & .308 \\
& Synonym       & .214 & .215 & .219 & \textbf{.000} & .241 \\
& Number swap   & .305 & .301 & .304 & \textbf{.156} & .469 \\
& Semantic sub. & .381 & .375 & .382 & \textbf{.290} & .468 \\
& Word order    & .226 & .213 & .220 & \textbf{.000} & .189 \\
\midrule
\multirow{6}{*}{\rotatebox{90}{Qwen 2.5}}
& Negation      & .136 & .131 & .133 & \textbf{.000} & .149 \\
& Typo          & .160 & .158 & .162 & \textbf{.000} & .144 \\
& Synonym       & .234 & .220 & .220 & \textbf{.002} & .177 \\
& Number swap   & .304 & .308 & .306 & \textbf{.005} & .218 \\
& Semantic sub. & .369 & .367 & .370 & \textbf{.001} & .373 \\
& Word order    & .105 & .103 & .098 & \textbf{.000} & .100 \\
\bottomrule
\end{tabular}
\end{center}
\end{table}

\section{Mean sensitivity control}
\label{app:sensitivity-control}

The per-family mean sensitivity values below check that the low-imbalance
solution is not obtained by suppressing all perturbation responses.

\begin{table}[h]
\centering
\caption{Mean relative sensitivity per family (standard vs
$V$-regularised SAE). The $V$-regularisation increases mean sensitivity
on most families across all three models, confirming that it
redistributes perturbation signals more uniformly rather than
suppressing them.}
\label{tab:sens}
\begin{center}
\small
\begin{tabular}{@{}llccc@{}}
\toprule
Model & Family & Standard & V-reg & Change \\
\midrule
\multirow{6}{*}{\rotatebox{90}{GPT-2}}
& Negation      & 0.073 & \textbf{0.115} & +58\% \\
& Typo          & 0.067 & \textbf{0.081} & +21\% \\
& Synonym       & 0.043 & \textbf{0.050} & +17\% \\
& Number swap   & 0.021 & 0.023 & +7\% \\
& Semantic sub. & 0.045 & 0.045 & 0\% \\
& Word order    & 0.043 & \textbf{0.057} & +34\% \\
\midrule
\multirow{6}{*}{\rotatebox{90}{Gemma 2}}
& Negation      & 0.342 & \textbf{0.614} & +79\% \\
& Typo          & 0.212 & \textbf{0.676} & +219\% \\
& Synonym       & 0.142 & \textbf{0.215} & +52\% \\
& Number swap   & 0.093 & \textbf{0.236} & +154\% \\
& Semantic sub. & 0.215 & \textbf{0.257} & +20\% \\
& Word order    & 0.224 & \textbf{0.291} & +30\% \\
\midrule
\multirow{6}{*}{\rotatebox{90}{Qwen 2.5}}
& Negation      & 0.379 & \textbf{0.507} & +34\% \\
& Typo          & 0.245 & \textbf{0.289} & +18\% \\
& Synonym       & 0.199 & \textbf{0.213} & +7\% \\
& Number swap   & 0.111 & \textbf{0.134} & +21\% \\
& Semantic sub. & 0.269 & \textbf{0.290} & +8\% \\
& Word order    & 0.302 & \textbf{0.358} & +18\% \\
\bottomrule
\end{tabular}
\end{center}
\end{table}

\section{Input-normalised gain and perturbation complexity}
\label{app:gain}

The response $D_i$ in Eq.~\ref{eq:response} does not account for the
magnitude of the input perturbation. The input-normalised gain divides the
feature-space response by the hidden-state response,
\begin{equation}
g_i = \frac{\|z_i^{\mathrm{pert}} - z_i^{\mathrm{orig}}\| /
            \|z_i^{\mathrm{orig}}\|}
           {\|h_i^{\mathrm{pert}} - h_i^{\mathrm{orig}}\| /
            \|h_i^{\mathrm{orig}}\|},
\end{equation}
and $V_{\mathrm{Gini}}$ of the $g_i$ measures whether the SAE applies a
similar amplification factor across all pairs under this normalisation.
The V-reg$_{\mathrm{n}}$
variant (trained with $\lambda_2 \cdot V_{\mathrm{gain}}$) achieves
$V_{\mathrm{gain}}$ values that round to 0.000 on all six families on all
three models (Table~\ref{tab:gain}), indicating very low gain imbalance
after input-complexity normalisation. The trade-off is expected. The two
variants optimise different objectives, with V-reg targeting the raw
response profile $D$ and V-reg$_{\mathrm{n}}$ targeting the
input-normalised profile $g$, so each minimises its own target and not
the other. Thus, V-reg$_{\mathrm{n}}$ does not minimise the raw $V$ (its
raw values are comparable to or higher than the standard SAE on some
families, Table~\ref{tab:comparison}). The two variants occupy
complementary points on the Pareto surface of perturbation awareness, and
the choice between them depends on whether raw or complexity-normalised
sensitivity is the quantity of interest.

\begin{table}[h]
\centering
\caption{$V_{\mathrm{Gini}}$ (gain) for the V-reg$_{\mathrm{n}}$ variant.
Values are rounded to three decimals and indicate very low gain imbalance
under input-complexity normalisation.}
\label{tab:gain}
\begin{center}
\small
\begin{tabular}{@{}lccc@{}}
\toprule
Family & GPT-2 & Gemma 2 & Qwen 2.5 \\
\midrule
Negation & 0.000 & 0.000 & 0.000 \\
Typo & 0.000 & 0.000 & 0.000 \\
Synonym & 0.000 & 0.000 & 0.000 \\
Number swap & 0.000 & 0.000 & 0.000 \\
Semantic sub. & 0.000 & 0.000 & 0.000 \\
Word order & 0.000 & 0.000 & 0.000 \\
\bottomrule
\end{tabular}
\end{center}
\end{table}

The sixteen families differ in input-level complexity, measured by the
coefficient of variation (CV) of hidden-state perturbation magnitudes and
by centered effective rank (Table~\ref{tab:complexity}). The full
geometry profile spans low-rank templatic families and more diffuse
medical-domain families. The fact that V-reg reaches a low-imbalance
regime despite this heterogeneity shows that it is not merely benefiting
from a homogeneous perturbation set.

\begin{table}[h]
\centering
\caption{Input perturbation complexity per family. Each entry reports
CV/centered effective rank for hidden-state perturbation magnitudes.}
\label{tab:complexity}
\begin{center}
\small
\begin{tabular}{@{}lccc@{}}
\toprule
Family & GPT-2 & Gemma 2 & Qwen 2.5 \\
\midrule
Anatomical dir. & 0.61/4.1 & 0.50/17.4 & 0.48/18.6 \\
Body part swap & 0.60/3.8 & 0.60/13.8 & 0.58/14.6 \\
Causal reversal & 0.70/5.2 & 0.50/18.9 & 0.65/16.6 \\
Condition neg. & 0.73/2.1 & 0.45/17.6 & 0.61/12.1 \\
Date/time change & 0.71/3.1 & 0.52/18.2 & 0.66/14.8 \\
Drug name swap & 0.56/4.1 & 0.32/23.0 & 0.46/18.5 \\
Frequency change & 1.13/2.6 & 0.63/14.6 & 0.71/12.6 \\
Gender swap & 0.83/2.7 & 0.29/18.1 & 0.21/20.7 \\
Negation & 0.52/2.4 & 0.20/16.3 & 0.26/15.7 \\
Number swap & 0.54/3.7 & 0.54/21.8 & 0.44/26.2 \\
Semantic sub. & 0.89/5.2 & 0.69/19.9 & 0.90/12.6 \\
Severity change & 0.85/2.9 & 0.54/18.2 & 0.69/15.3 \\
Synonym & 0.65/2.6 & 0.27/36.7 & 0.33/34.4 \\
Typo & 0.61/5.1 & 0.27/36.6 & 0.29/35.8 \\
Unit measure & 0.61/4.1 & 0.60/13.4 & 0.72/10.8 \\
Word order & 0.67/2.4 & 0.40/28.5 & 0.21/35.7 \\
\bottomrule
\end{tabular}
\end{center}
\end{table}

\section{Reconstruction quality across methods}
\label{app:recon}

\begin{table}[h]
\centering
\caption{Reconstruction MSE and sparsity (fraction of active features)
across five SAE methods and three models. All share
$d_{\mathrm{sae}} = 4{,}096$ and 5{,}000 steps.}
\label{tab:recon}
\begin{center}
\small
\begin{tabular}{@{}llcc@{}}
\toprule
Model & Method & MSE & Frac active \\
\midrule
\multirow{5}{*}{\rotatebox{90}{GPT-2}}
& Standard & 0.0034 & 0.352 \\
& JumpReLU & 0.0035 & 0.333 \\
& MDL & 0.0065 & 0.337 \\
& V-reg & 0.0069 & 0.344 \\
& V-reg$_{\mathrm{n}}$ & 0.0033 & 0.329 \\
\midrule
\multirow{5}{*}{\rotatebox{90}{Gemma 2}}
& Standard & 0.0130 & 0.072 \\
& JumpReLU & 0.0222 & 0.097 \\
& MDL & 0.0250 & 0.076 \\
& V-reg & 0.0119 & 0.094 \\
& V-reg$_{\mathrm{n}}$ & 0.0142 & 0.071 \\
\midrule
\multirow{5}{*}{\rotatebox{90}{Qwen 2.5}}
& Standard & 0.0052 & 0.172 \\
& JumpReLU & 0.0020 & 0.204 \\
& MDL & 0.0038 & 0.248 \\
& V-reg & 0.0104 & 0.162 \\
& V-reg$_{\mathrm{n}}$ & 0.0040 & 0.160 \\
\bottomrule
\end{tabular}
\end{center}
\end{table}

On GPT-2, the V-reg MSE (0.0069) is about twice the standard (0.0034), a
modest cost for a $10\times$ reduction in $V$. On Gemma~2, V-reg achieves
the lowest MSE of all five methods (0.0119 vs standard 0.0130), indicating
that perturbation-aware orientation can sometimes improve reconstruction.
On Qwen~2.5, V-reg has the highest MSE (0.0104), but the absolute level
remains low. The V-reg$_{\mathrm{n}}$ variant consistently achieves
reconstruction comparable to or better than the standard, because the
gain-normalised $V$ loss places less pressure on the encoder weights.

\section{Hidden-state extraction protocol}
\label{app:extraction}

Decoder-only language models can use different padding conventions during
batched evaluation. To compare perturbation responses at matched sequence
positions, all hidden-state caches in the evaluation use a
true-last-token extraction rule. For each sequence, the extraction index is
the final position with attention-mask value one. This rule selects the
last non-padding token independently of whether the tokenizer pads on the
left or on the right.

This convention is used for GPT-2, Gemma~2, and Qwen~2.5 perturbation
caches. It makes the response quantity in Eq.~\eqref{eq:response}
well-defined across models, because $z_i^{\mathrm{orig}}$ and
$z_i^{\mathrm{pert}}$ are computed from the same semantic sequence
position in the original and perturbed inputs. The same cached hidden
states are then reused for Standard, JumpReLU, MDL, and
$V$-regularised SAE evaluation under the frozen multi-seed protocol.

\section{Reporting rounded small values}
\label{app:nearzero}

Several tables display $V_{\mathrm{Gini}}$ values rounded to three
decimals. An entry printed as 0.000 should be read as a rounded small
value, not as a mathematical claim that the underlying response or
imbalance is exactly zero. The full-precision values are retained in the
machine-readable result files. This convention is especially important for
low-imbalance summaries. A rounded value indicates that the response
profile is close to uniform at the displayed precision; the
scale-sensitive quantities in \S\ref{sec:measure} report response
magnitude. Two factors make such values more likely. The $\varepsilon$ in
Eq.~\eqref{eq:response} stabilises small denominators and compresses very
small responses toward a common scale, and for small or internally similar
families the sorted response profile can approach uniformity, which drives
the Gini statistic below the displayed precision. Near-zero cells are
therefore accompanied by seed dispersion (Table~\ref{tab:e7}) and by the
family-geometry controls of Appendix~\ref{app:family-geometry}, so that a
near-zero $V_{\mathrm{Gini}}$ is not read as a trivial saturation effect.

\section{Family-geometry and template-cluster controls}
\label{app:family-geometry}

The perturbation set contains repeated surface templates, so uncertainty
should not rely only on pair-level resampling. We infer exact-skeleton
template clusters by replacing the changed span in each original/perturbed
pair with a slot marker. Across the 600 perturbation pairs, this procedure
finds 424 template clusters, of which 350 are singleton templates and 74
are multi-pair templates. In total, 250/600 pairs (41.7\%) belong to
repeated templates. Repeated exact surface templates occur in five
families, namely \texttt{negation}, \texttt{number\_swap},
\texttt{semantic\_substitution}, \texttt{synonym}, and \texttt{typo}. The
other eleven families are exact-skeleton singleton families, so
template-cluster resampling coincides with pair-level resampling for those
families.

Table~\ref{tab:app-template-robustness} summarises the main resampling
checks. The fixed-family intervals resample template clusters within each
family while keeping the sixteen evaluated families fixed. The
family-resampled intervals additionally resample families. The latter are
wider, as expected, but remain positive for the relative lower-tail effect
and for the absolute code-displacement lower tail.

\begin{table}[h]
\centering
\caption{Template-cluster resampling summary. LOTO denotes
leave-one-template-out, and LOFO denotes leave-one-family-out.}
\label{tab:app-template-robustness}
\begin{center}
\small
\begin{tabular}{@{}lccc@{}}
\toprule
Model & LOTO positive templates & LOFO $\Delta L_{20}(D)$ range &
Centered effective-rank range \\
\midrule
Gemma~2 & $424/424$ & $+0.1149$--$+0.1401$ & $13.4$--$36.7$ \\
GPT-2 & $424/424$ & $+0.0522$--$+0.0584$ & $2.1$--$5.2$ \\
Qwen~2.5 & $424/424$ & $+0.1140$--$+0.1372$ & $10.8$--$35.8$ \\
\bottomrule
\end{tabular}
\end{center}
\end{table}

The hidden-state perturbation geometry differs across models. GPT-2 has a
lower-dimensional and more direction-concentrated perturbation geometry
than Gemma~2 or Qwen~2.5, but its sparse-code lower-tail response still
improves under $V$-regularisation. Family-size ablations also support the
same conclusion. At $K=10$ pairs per family, all relative lower-tail
effects remain positive in all repeats for Gemma~2 and GPT-2, and the
worst Qwen~2.5 family is positive in 99.9\% of repeats. Absolute
code-distance effects are more heterogeneous at small $K$, especially for
Qwen~2.5, but the aggregate cluster-aware intervals remain positive.
Decoded-response lower-tail effects are stable for Gemma~2 and Qwen~2.5,
while GPT-2 remains near zero after decoding.

\section{Sixteen-family per-family values}
\label{app:joint16-detail}

The per-family values below provide the detailed support for the
sixteen-family summary in Table~\ref{tab:joint16}.

\begin{table}[h]
\centering
\caption{Per-family $V_{\mathrm{Gini}}$ (raw) for the sixteen-family joint
$V$-regularised SAE on GPT-2 and Qwen~2.5, standard SAE versus V-reg. The
six families above the rule are the original families; the ten below are
the auto-generated medical-domain families, none of which informed the
metric or training protocol. Both groups reach a low-imbalance regime in
this descriptive comparison. Source files
\texttt{results/joint\_eval\_\{gpt2,qwen-2.5-3b\}.json}.}
\label{tab:joint16detail}
\begin{center}
\small
\begin{tabular}{@{}l cc c cc@{}}
\toprule
& \multicolumn{2}{c}{GPT-2 small} & & \multicolumn{2}{c}{Qwen 2.5 3B} \\
\cmidrule(lr){2-3}\cmidrule(lr){5-6}
Family & Std & V-reg & & Std & V-reg \\
\midrule
Negation                & 0.219 & \textbf{0.000} & & 0.136 & \textbf{0.000} \\
Typo                    & 0.307 & \textbf{0.000} & & 0.160 & \textbf{0.000} \\
Synonym                 & 0.307 & \textbf{0.000} & & 0.234 & \textbf{0.000} \\
Number swap             & 0.264 & \textbf{0.000} & & 0.303 & \textbf{0.000} \\
Semantic substitution   & 0.435 & \textbf{0.035} & & 0.368 & \textbf{0.000} \\
Word order              & 0.281 & \textbf{0.001} & & 0.105 & \textbf{0.000} \\
\addlinespace
\textit{6-family mean}  & \textit{0.302} & \textit{0.006} & & \textit{0.218} & \textit{0.000} \\
\midrule
Anatomical direction    & 0.279 & \textbf{0.000} & & 0.225 & \textbf{0.000} \\
Body-part swap          & 0.238 & \textbf{0.000} & & 0.258 & \textbf{0.000} \\
Causal reversal         & 0.352 & \textbf{0.000} & & 0.291 & \textbf{0.000} \\
Condition negation      & 0.294 & \textbf{0.000} & & 0.192 & \textbf{0.000} \\
Date/time change        & 0.350 & \textbf{0.000} & & 0.294 & \textbf{0.000} \\
Drug-name swap          & 0.266 & \textbf{0.000} & & 0.183 & \textbf{0.000} \\
Frequency change        & 0.377 & \textbf{0.000} & & 0.341 & \textbf{0.000} \\
Gender swap             & 0.337 & \textbf{0.000} & & 0.111 & \textbf{0.000} \\
Severity change         & 0.428 & \textbf{0.000} & & 0.320 & \textbf{0.000} \\
Unit of measure         & 0.296 & \textbf{0.000} & & 0.299 & \textbf{0.000} \\
\addlinespace
\textit{10-family medical-domain mean} & \textit{0.322} & \textit{0.000} & & \textit{0.251} & \textit{0.000} \\
\midrule
\textbf{All-16 mean}    & \textbf{0.314} & \textbf{0.002} & & \textbf{0.239} & \textbf{0.000} \\
\bottomrule
\end{tabular}
\end{center}
\end{table}

\section{Gemma extraction consistency}
\label{app:gemma-provenance}

Decoder-only models differ in their padding conventions, and under
Gemma's left-padding convention a fixed-index extraction rule selects a
padding position rather than the final token, so caches built with such a
rule are not computed at the true last token. All Gemma~2 analyses in this
paper therefore use the attention-mask-based true-last-token rule of
Appendix~\ref{app:extraction}, the same rule applied to GPT-2 and
Qwen~2.5. Under this rule Gemma follows the same qualitative pattern as
the other two models in the primary response-profile metrics. An apparent
imbalance plateau that arises under fixed-index extraction on Gemma is an
artefact of the misaligned token position rather than a structural
property of the Gemma hidden state, which is why the position-level rule
matters for cross-model comparability. The scientific interpretation is
based on the true-last-token caches and the frozen cross-model evaluation
reported in Table~\ref{tab:e7}.

\section{Pre-caching of perturbation hidden states}
\label{app:lambda}

The $V$ term operates on deterministic, pre-cached perturbation hidden
states. Each perturbation pair is passed through the frozen language model
once, stored, and reused at every SAE training step. This mirrors standard
activation caching for the reconstruction loss and ensures that changes
in the perturbation-response objective come from SAE optimisation rather
than from stochastic variation in language-model forward passes. For
production use, $\lambda_2$ should be selected on a validation set via the
Pareto frontier of reconstruction MSE versus mean $V_{\mathrm{Gini}}$.

\section{Perturbation family descriptions}
\label{app:families}

The six original families contain 50 sentence pairs each; the ten
medical-domain families (used in the sixteen-family study of
\S\ref{sec:res-joint16}) contain 30 pairs each, for 600 pairs in total.
Representative examples for every family follow; the complete datasets are
provided as supplementary material (see \nameref{sec:repro}).

\subsection{Original families}

\paragraph{Semantic substitution.}
``The cat sat on the mat\ldots'' $\to$ ``The dog sat on the mat\ldots'';
``\ldots very hot and sunny\ldots'' $\to$ ``\ldots very cold and sunny\ldots'';
``\ldots drove her car\ldots'' $\to$ ``\ldots drove her truck\ldots'';
``The doctor examined\ldots'' $\to$ ``The nurse examined\ldots'';
``\ldots from the small shop'' $\to$ ``\ldots from the small store''.

\paragraph{Typo (adjacent-character transposition).}
``\ldots admitted to the hospital\ldots'' $\to$ ``\ldots hopsital\ldots'';
``\ldots published the results\ldots'' $\to$ ``\ldots publsihed\ldots'';
``\ldots an important message\ldots'' $\to$ ``\ldots mesasge\ldots'';
``\ldots their assignment\ldots'' $\to$ ``\ldots thier\ldots'';
``\ldots explained the concept\ldots'' $\to$ ``\ldots explaind\ldots''.

\paragraph{Negation.}
``The test result is positive\ldots'' $\to$ ``\ldots is not positive\ldots'';
``The drug is safe\ldots'' $\to$ ``\ldots is not safe\ldots'';
``The system is working\ldots'' $\to$ ``\ldots is not working\ldots'';
``The patient is responding\ldots'' $\to$ ``\ldots is not responding\ldots'';
``The data is consistent\ldots'' $\to$ ``\ldots is not consistent\ldots''.

\paragraph{Synonym.}
``\ldots very big\ldots'' $\to$ ``\ldots very large\ldots'';
``\ldots happy to hear\ldots'' $\to$ ``\ldots glad to hear\ldots'';
``\ldots very hard to complete\ldots'' $\to$ ``\ldots very difficult\ldots'';
``She started working\ldots'' $\to$ ``She began working\ldots'';
``\ldots fix the broken machine\ldots'' $\to$ ``\ldots repair\ldots''.

\paragraph{Number swap.}
``\ldots received 3 doses\ldots'' $\to$ ``\ldots 8 doses\ldots'';
``\ldots 5 hospitals\ldots'' $\to$ ``\ldots 9 hospitals\ldots'';
``\ldots 2 pills\ldots'' $\to$ ``\ldots 7 pills\ldots'';
``\ldots scored 90 points\ldots'' $\to$ ``\ldots 40 points\ldots'';
``\ldots reached 30 degrees\ldots'' $\to$ ``\ldots 80 degrees\ldots''.

\paragraph{Word order (adjacent-word swap).}
``The black cat\ldots'' $\to$ ``The cat black\ldots'';
``She quickly finished\ldots'' $\to$ ``She finished quickly\ldots'';
``The old man\ldots'' $\to$ ``The man old\ldots'';
``The bright sun\ldots'' $\to$ ``The sun bright\ldots'';
``The young student\ldots'' $\to$ ``The student young\ldots''.

\subsection{Additional medical-domain families}

These ten auto-generated families were not used to design the metric or
training protocol; the sixteen-family study
(\S\ref{sec:res-joint16}, Table~\ref{tab:joint16detail}) evaluates
generalisation to them. Three representative pairs per family follow.

\paragraph{Anatomical direction (anterior$\leftrightarrow$posterior, proximal$\leftrightarrow$distal, dorsal$\leftrightarrow$ventral).}
``The mass is located in the \emph{anterior} mediastinum'' $\to$
``\ldots \emph{posterior} mediastinum'';
``Pain is felt in the \emph{proximal} part of forearm'' $\to$
``\ldots \emph{distal} part\ldots'';
``The lesion is on the \emph{dorsal} surface of hand'' $\to$
``\ldots \emph{ventral} surface\ldots''.

\paragraph{Body-part swap (left$\leftrightarrow$right).}
``\ldots pain in the \emph{left} knee'' $\to$ ``\ldots \emph{right} knee'';
``Fracture detected in the \emph{right} femur'' $\to$ ``\ldots \emph{left} femur'';
``Swelling observed in the \emph{left} ankle'' $\to$ ``\ldots \emph{right} ankle''.

\paragraph{Causal reversal.}
``The fever was caused by the infection'' $\to$
``The infection was caused by the fever'';
``Pain \emph{increased} after taking the medication'' $\to$
``Pain \emph{decreased}\ldots'';
``\ldots resulted in significant \emph{improvement}'' $\to$
``\ldots significant \emph{deterioration}''.

\paragraph{Condition negation.}
``The patient \emph{has} a history of diabetes'' $\to$
``\ldots \emph{has no} history of diabetes'';
``Allergies to penicillin \emph{are} documented'' $\to$
``\emph{No} allergies\ldots documented'';
``There \emph{is} evidence of metastatic disease'' $\to$
``There \emph{is no} evidence\ldots''.

\paragraph{Date/time change.}
``\ldots scheduled for \emph{Monday} morning'' $\to$ ``\ldots \emph{Friday} morning'';
``\ldots admitted on \emph{January} 15'' $\to$ ``\ldots \emph{March} 15'';
``Follow-up in \emph{2} weeks'' $\to$ ``\ldots \emph{6} weeks''.

\paragraph{Drug-name swap.}
``\ldots prescribed \emph{metformin} daily'' $\to$ ``\ldots \emph{lisinopril} daily'';
``Start \emph{ibuprofen} 400\,mg\ldots'' $\to$ ``Start \emph{acetaminophen} 400\,mg\ldots'';
``\ldots taking \emph{warfarin} for clots'' $\to$ ``\ldots \emph{aspirin} for clots''.

\paragraph{Frequency change.}
``Take the medication \emph{once} daily'' $\to$ ``\ldots \emph{three times} daily'';
``Apply the cream \emph{twice daily}'' $\to$ ``\ldots \emph{once weekly}'';
``\ldots given every \emph{two weeks}'' $\to$ ``\ldots every \emph{two months}''.

\paragraph{Gender swap (he$\leftrightarrow$she, his$\leftrightarrow$her).}
``\ldots reported \emph{he} felt dizzy'' $\to$ ``\ldots \emph{she} felt dizzy'';
``\emph{She} was admitted to the ward'' $\to$ ``\emph{He} was admitted\ldots'';
``\emph{His} blood pressure was elevated'' $\to$ ``\emph{Her} blood pressure\ldots''.

\paragraph{Severity change (mild$\to$severe, minor$\to$major).}
``\ldots \emph{mild} chest pain'' $\to$ ``\ldots \emph{severe} chest pain'';
``\ldots \emph{moderate} difficulty breathing'' $\to$ ``\ldots \emph{extreme} difficulty\ldots'';
``\ldots \emph{slight} swelling'' $\to$ ``\ldots \emph{significant} swelling''.

\paragraph{Unit of measure (mg$\to$g, mL$\to$L, cm$\to$mm).}
``The dosage is 500\,\emph{mg}\ldots'' $\to$ ``\ldots 500\,\emph{g}\ldots'';
``Inject 0.5\,\emph{mL}\ldots'' $\to$ ``Inject 0.5\,\emph{L}\ldots'';
``\ldots measured 2.3\,\emph{cm}\ldots'' $\to$ ``\ldots 2.3\,\emph{mm}\ldots''.

\end{document}